\pdfoutput=1
\documentclass[letterpaper]{article}
\usepackage{aaai}
\usepackage{times}
\usepackage{helvet}
\usepackage{courier}

\usepackage{epsfig}
\usepackage{graphicx}
\usepackage{amsmath}
\usepackage{amssymb}
\usepackage{caption}
\usepackage{subcaption}
\usepackage{xcolor}
\usepackage{amsthm}
\usepackage{thmtools, thm-restate}
\newtheorem{theorem}{Theorem}[]

\newtheorem{lemma}[theorem]{Lemma}

\newcommand{\etal}{{et al}.~}
\newcommand{\ie}{\textit{i}.\textit{e}.~}
\newcommand{\eg}{\textit{e}.\textit{g}.~}
\newcommand{\centered}[1]{\begin{tabular}{l} #1 \end{tabular}}

\begin{document}
%
\title{Spectral Analysis for Semantic Segmentation with Applications on Feature Truncation and Weak Annotation}
\author{Li-Wei Chen$^\dagger$, Wei-Chen Chiu$^\dagger$, Chin-Tien Wu$^{\dagger \ddagger}$ \\
$\dagger$ Department of Computer Science, National Yang Ming Chiao Tung University, Hsin-Chu, Taiwan \\
$\ddagger$ Department of Applied Mathematics, National Yang Ming Chiao Tung University, Hsin-Chu, Taiwan \\
{lwchen6309.cs07@nycu.edu.tw, walon@cs.nctu.edu.tw, ctw@math.nctu.edu.tw}
}

\maketitle

\begin{abstract}
It is well known that semantic segmentation neural networks (SSNNs) produce dense segmentation maps to resolve the objects' boundaries while restrict the prediction on down-sampled grids to alleviate the computational cost. A striking balance between the accuracy and the training cost of the SSNNs such as U-Net exists. We propose a spectral analysis to investigate the correlations among the resolution of the down sampled grid, the loss function and the accuracy of the SSNNs. By analyzing the network back-propagation process in frequency domain, we discover that the traditional loss function, cross-entropy, and the key features of CNN are mainly affected by the low-frequency components of segmentation labels. Our discoveries can be applied to SSNNs in several ways including (i) determining an efficient low resolution grid for resolving the segmentation maps (ii) pruning the networks by truncating the high frequency decoder features for saving computation costs, and (iii) using block-wise weak annotation for saving the labeling time. Experimental results shown in this paper agree with our spectral analysis for the networks such as DeepLab V3+ and Deep Aggregation Net (DAN).

\end{abstract}

\section{Introduction}\label{sec:intro}
Semantic segmentation, which densely assigns semantic labels to each image pixel, is one of the important topics in computer vision. Recently we have witnessed that the CNNs based on the encoder-decoder architecture~\cite{long2015fully,chen2018encoder,badrinarayanan2017segnet,zhao2017psp,chen2014semantic,yang8578486,noh7410535,yu2015multi} achieve striking performance on several segmentation benchmarks~\cite{mottaghi2014role,everingham2015pascal,cordts2016cityscapes,zhou2017ade,caesar2018coco}. 
The existing works, \eg Fully Convolutional Neural Network (FCN)~\cite{long2015fully}, U-Net \cite{ronneberger2015u}, and DeepLab models~ \cite{chen2014semantic,chen2017rethinking,chen2018encoder}, utilize encoder-decoder architecture to resolve the dense map. 
Generally, an encoder-decoder architecture consists of an encoder module that gradually reduces the spatial resolution of features for extracting context information and a decoder module that aggregates the information from the encoder and recovers the spatial resolution of the dense segmentation map. On the other hand, these networks predict the dense segmentation map on a \textbf{low-resolution grid (LRG)}, \eg $\frac{1}{4}$ and $\frac{1}{8}$ of the original image resolution, which is then up-sampled to the original image resolution, to save the computational cost~\cite{chen2018encoder,long2015fully}.
Moreover, to save the labeling cost, some networks can learn sufficient semantic contents from the weak annotations such as the annotation with coarse contours of objects~\cite{papandreou2015weakly,dai2015boxsup,khoreva2017simple} or those annotations derived from the image level ~\cite{ahn2018learning,jing2019coarse,zhou2019wails,shimoda2019self,sun2019fully,nivaggioli2019weakly}. 
Despite the labeling inaccuracy near object boundaries, the existing works demonstrate that networks can still achieve comparable accuracy with those networks that learn from the pixel-wise groundtruth annotation. 
Weak annotations such as those based on the coarse contour are indeed similar to the annotation obtained by down-sampling the pixel-wise groundtruth annotation onto the LRG. 
Regardless of the success of using LRG for prediction and the learning from weak annotations for cost saving, it yet remains unclear how the LRG and weak annotations affect the accuracy of SSNNs in theory. 

Recently, ~\cite{rahaman2018spectral,ronen2019convergence,luo2019theory,yang2019fine,xu2019frequency} have demonstrated that the trained networks tend to learn the low-frequency component of target signals in the regression of the uniformly distributed data. Such tendency is the so-called spectral bias. In SSNNs, LRG and weak annotation resemble the low-frequency component of segmentation maps. On the other hand, the labeling inaccuracy near object boundaries is more likely to be in the high-frequency range. Is it true that the SSNNs also have spectral bias? Is the spectral bias the reason why the SSNNs can maintain their accuracy when the LRG and weak annotation are applied? The answers to these questions are not trivial especially since the distribution of segmentation annotation is not necessarily uniformly distributed along frequency regime. In this work, we investigate the correlations among the frequency distributions of ground truth annotation, network output segmentation, and the evaluation metric (\eg intersection-over-union (IoU) score) in frequency domain. We present a theoretical analysis on the influence of the objective function (\eg cross-entropy (CE)), CNN features, and the LRG on the network prediction by using the spectral analysis.

\noindent From our theoretical results, we identify important features of the segmentation network and summarize the following key observations:
\begin{itemize}
\item The cross-entropy (CE) can be explicitly decomposed into the summation of frequency components. We find that CE is mainly contributed by the low-frequency component of segmentation maps.
\item Spectral analysis of IoU score reveals its close relation to the CE. These results justify that the segmentation networks are trained upon CE while its performance is
evaluated upon IoU. Additionally, we discover boundary IoU, a metric specialized for object boundary, is also mainly contributed by the low-frequency component of segmentation maps.
\item The correlation between the segmentation logits and the features within CNNs in the frequency domain shows that the segmentation logits of a specific frequency are mainly affected by the features within the same frequency.
\item Based on the findings above, truncating high frequency components of smooth features does not interfere the performance of semantic segmentation networks.
\end{itemize}
\noindent Our findings above contribute to the semantic segmentation networks in the following two objectives:
\begin{enumerate}
    \item \textbf{Feature truncation for segmentation networks.} One can determine the proper size of LRG from spectral analysis. The features in the decoder, which are generally smoother than these in the encoder, can then be truncated to reduce the computational cost. This truncation method can be integrated with the commonly-used pruning approaches~\cite{liu2018rethinking,he2019asymptotic} for further cost reduction. We save nearly 70\% cost of SSNNs in the optimal case of our experiments. 
    \item \textbf{Block-wise annotation.} For semantic segmentation, it is easier to collect a weak annotation that keeps the low-frequency information of the full pixel-wise ground truth annotation. Here, we propose a block-wise annotation that emulates the weak annotations, where only the coarse contours of the instances in the segmentation map~\cite{papandreou2015weakly,khoreva2017simple} are used, and show that the segmentation networks trained via these block-wise annotations are still efficient and accurate. 
\end{enumerate}
This paper is organized as follows. We present the proposed analysis and validate our observations in section~\ref{sec:method}. Applications shown in section~\ref{sec:experiments} demonstrate the contributions of our spectral analysis to SSNNs. The experimental results are expected in our analysis. We relegate the review of related work such as segmentation network, \ie SSNN, network pruning, and spectral bias in section~\ref{ssec:related work} of the appendix. Finally, we draw our conclusions in section ~\ref{sec:conclusion}

\section{Proposed Spectral Analysis}\label{sec:method}

To clearly depict the efficacy of LRG, in section \ref{ssec:method_spectral_ce}, we analyze the formalism of cross-entropy (CE) objective function and the intersection-over-union (IoU) evaluation metric in frequency domain. The detailed analysis of IoU can be seen in section~\ref{ssec:method_spectral_ioU} of appendix.
Our results show that CE can be decomposed into the components of frequencies and positively correlate to IoU score; this justifies the usage of CE as objective function and IoU as evaluation metric in the network training framework. 
Moreover, we investigate the learning mechanism in convolutional layers by deducing the gradient propagation in frequency domain. We demonstrate the correlation between the segmentation output and the features in CNNs in section \ref{ssec:method_spec_grad}.  
Our results suggest that the high-frequency components of the features and annotations have less influence on the performance of the segmentation networks due to the band limit introduced by the LRG. 
\\
\noindent\textbf{Notation}. 
The notations in this section are defined as follows. In general, the upper case letters, \eg $X, Y, Z, B$, denote the functional in the spatial domain while the lower case letters \eg $x, y, z, b$, denote the corresponding spectrum in the frequency domain $\nu$. For example, the spectrum $y(\nu)=\mathcal{F}(Y(t))$ where  $\mathcal{F}$ is the Fourier transform operator. The rest of notations will be defined whenever they appear.

\subsection{Spectral Analysis on Cross-Entropy and IoU}\label{ssec:method_spectral_ce}
Let $Y(t,c)$ denote the segmentation logits produced by a semantic segmentation network and $B(t,c)$ denote the groundtruth annotation, in which $c$ and $t$ are indexes of the object class and image pixel of the input. respectively. The commonly-used objective function for learning semantic segmentation, cross-entropy (CE), can be written as
\begin{equation}
    \begin{aligned}
         \mathcal{L}_{CE} & = -  { \sum _{c}} {\int} B(t,c) {log} (   { \frac {e^{Y(t,c)}}{{ \sum _{c}} e^{Y(t,c)}}}) dt 
        \\ & = - { \sum _{c}} {\int} B(t,c)(Y(t,c) - {Y_{p}}(t))dt,
    \end{aligned}
    \label{Eq:ce}
\end{equation}
where ${Y_{p}}(t)={log}({ \sum _{c}} e^{Y(t,c)})$. Transforming this integral into frequency domain $\nu$ gives theorem~\ref{theorem:ce}. See proof of theorem~\ref{theorem:ce} in section~\ref{ssec:proof_spectral_ce} of appendix.
\begin{restatable}[Spectral decomposition of Cross-Entropy]{theorem}{ce}
\label{theorem:ce}
Given $Y$ the segmentation logits and $B$ the groundtruth annotation, the cross-entropy $\mathcal{L}_{CE}$ can be decomposed as 
\begin{equation}
    \begin{aligned}
        \\ \mathcal{L}_{CE} & = { \sum _{\nu}} { \sum_{c} }{b}(-\nu, c)({y_{p}}(\nu) - y(\nu, c))
        \\ & = { \sum _{\nu}} \mathcal{L}_{ce}(\nu),
        \\ \mathcal{L}_{ce}(\nu) &= {\sum_{c}}{b}(-\nu, c)({y_{p}}(\nu) - y(\nu, c))
    \end{aligned}
    \label{Eq:spec_ce_component}
\end{equation}
where $b, y,$ and $y_{p}$ are the spectra of $B, Y,$ and $Y_{p}$, respectively; $\mathcal{L}_{ce}(\nu)$ is defined as the \textbf{frequency components of CE}. 
\end{restatable}

The contribution from each frequency component to CE can thus be evaluated. Let $\widehat{y}(\nu,c) = {y_{p}}(\nu) - y(\nu, c)$. It follows immediately that $\mathcal{L}_{ce}(\nu)$ is small when either $\widehat{y}(\nu,c)$ or $b(\nu,c)$ is small. Recalls that, as mentioned in section~\ref{sec:intro}, $Y(t,c)$ is generally predicted upon the low-resolution grid. Hence, $\widehat{y}(\nu_i,c)$ as well as $\mathcal{L}_{ce}(\nu_i)$ should be small at high-frequency $\nu_i$. 
Besides, the magnitude of high-frequency components of the ground truth annotation is generally small dues to the intrinsic smoothness of segmentation map~\footnote{Segmentation map usually consists of several segments for each class. All the internal region of the segment is flat while only the regions near the segmentation boundaries contain high-frequency components.}.
One can therefore conclude that $\mathcal{L}_{ce}$ is mainly contributed by the low-frequency components. This conclusion will be validated later in section~\ref{ssec:exp_spec_ce_grad:ce} and leads to our \textcolor{green}{first key observation}.

Furthermore, by extending the definition of IoU to a continuous space, the IoU can be represented in frequency domain as 
\begin{equation}
    \begin{aligned}
    IoU({s,b}) =  \frac{1}{{\frac{{ {s(0) + b(0)} }}{{{\int} s(\nu)b(- \nu)d\nu}} - 1}},
    \end{aligned}
    \label{Eq:iou_ft_brief}
\end{equation}
where $S(t,c)={ \frac {e^{Y(t,c)}}{{ \sum _{\tilde{c}}} e^{Y(t,\tilde{c})}}}$ and $s$ is the spectra of $S$; we skip $c$ for simplicity. The detailed derivation of Eq~\ref{Eq:iou_ft_brief} are illustrated in section~\ref{ssec:method_spectral_ioU} of appendix. 
There, we demonstrate that the $IoU(s,b)$ is positively correlates to the CE in Eq.~\ref{Eq:ce}. Since $\mathcal{L}_{ce}$ is mainly contributed by low-frequency components as mentioned above, so does the $IoU(s,b)$. Moreover, we also analyze the boundary IoU~\cite{cheng2021boundary} in section~\ref{ssec:method_spectral_bioU} of appendix. The boundary IoU focuses on the segmentation evaluation measure near object boundaries. Our analytic result suggests that the boundary IoU is also mainly contributed by low-frequency components of the segmentation map $S$. 
This further demonstrates the boundary IoU is not only sensitive to object boundary, as highlighted in~\cite{cheng2021boundary}, but also also sensitive to smooth region, \ie low-frequency component, thus provides an comprehensive measure of segmentation map.
Hence, we conclude that IoU and boundary IoU are mainly contributed by low-frequency components. This corresponds to our \textcolor{green}{second key observation}. In section ~\ref{ssec:exp_blk_annot}, we validate this observation by showing how the IoU is changed by using LRG and WA which truncate the high-frequency components indeed.

From our first and second observations, the common learning procedure (the networks are trained by the CE objective function while being validated by the IoU scores) for SSNNs is rationalized. Later, we adopt the decomposition of CE to study the frequency response and take IoU as a reasonable metric for evaluation in our experiments and analysis. On the other hand, training the SSNNs using WA, which drops the high-frequency component of segmentation maps, seems to be a cost-effective approach compared to the training using expensive pixel-wise groundtruth annotation.
Most recent works in studying WA ~\cite{papandreou2015weakly,dai2015boxsup,khoreva2017simple,GUO2021108063,LU2021107924,ahn2018learning,jing2019coarse,zhou2019wails,shimoda2019self,sun2019fully,nivaggioli2019weakly} demonstrate efficacy of WA by experiments on benchmark datasets. Here, our spectral analysis provides a theoretical justification for the usage of WA.

\subsection{Spectral Gradient of Convolutional Layers}\label{ssec:method_spec_grad}
In this section, we further reveal how $\mathcal{L}_{ce}(\nu)$ updates the network by deducing the gradient propagation of $\mathcal{L}_{ce}(\nu)$ within CNNs, especially for the gradient in low-frequency regime. 
Hereafter, we refer the gradient propagation with respect to input feature $X$ in frequency domain as the \textbf{spectral gradient (SpG)}. With SpG, we would like to analyze how the low-frequency feature of SSNNs affect the network performance. For simplicity, we deduce the SpG for a single convolutional layer, including a convolution and an activation function
\begin{equation}
    \begin{aligned}
        Z(t) & = K(t) \otimes X(t)
        \\ Y(t) & =\sigma(Z(t)),
    \end{aligned}
    \label{Eq:conv_layer}
\end{equation} 
where $K$ is the kernel, $X$ is the the input feature, $Y$ is the output of convolutional layer and $\sigma$ is the differentiable soft-plus activation function $\sigma(Z(t)) = \log ({1+{e^{Z(t)}}})$. 
The SpG for a convolutional layer is given as below.
\begin{restatable}[The spectral gradient for a convolutional layer]{theorem}{ce_grad}
\label{theorem:ce_grad}
Let $B$ be the groundtruth annotation and $X$, $K$ and $Y$ be as in Eq.~\ref{Eq:conv_layer}. Assume $K(t)$ is small and $|X(t)| < 1$ \footnote{These assumptions rely on the fact that the numeric scale of feature and kernel are usually limited to a small range of value for the numeric stability of networks.} 
The spectral gradient of the output $y$ is
\begin{equation}
    \begin{aligned}
         \frac{\partial y(\nu_i)}{\partial x(\nu_j)} \simeq \frac{1}{2}k( \nu_j){{\delta}_{\nu_j}(\nu_i)},
    \end{aligned}
\label{Eq:convlayer_spec_grad}
\end{equation}
and the spectral gradient of $\mathcal{L}_{ce}(\nu)$ is 
\begin{equation}
    \begin{aligned}
        \frac{{\partial \mathcal{L}_{ce}(\nu_i)}}{{\partial x(\nu_j)}} 
         \simeq & \sum _c \frac{1}{2}k(\nu_j,c)[D_{0}(\nu_i) s(-\nu_j, c) 
         \\ &- \delta_{\nu_j}(\nu_i)b(-\nu_i,c)],
    \end{aligned}
\label{Eq:ce_spec_grad}
\end{equation}
where $y(\nu) = {\cal F}(Y(t))$, $b(\nu) = {\cal F}(B(t))$, $k(\nu) = {\cal F}(K(t))$, $x(\nu) = {\cal F}(X(t))$, and $s(\nu) = {\cal F}(S(t))$ the spectrum of segmentation output. ${\delta}_{\nu_j}(\nu_i)$ is the Kronecker delta function, which is here after abbreviated as the delta function, and $D_{0}(\nu_i)$ is the Dirac delta function. 
\end{restatable}
We relegate the proof of Eq.~\ref{Eq:convlayer_spec_grad} and Eq.~\ref{Eq:ce_spec_grad} to lemma~\ref{lemma:conv_grad} in appendix~\ref{ssec:sgd_conv} and lemma~\ref{lemma:ce_grad} in appendix~\ref{ssec:sgd_ce}, respectively.
Clearly from the SpG in Eq.~\ref{Eq:convlayer_spec_grad}, $y(\nu_i)$ is affected only by $x(\nu_i)$ with the same frequency. \textcolor{green}{This is our third key observation.}\\
Next, let us further analyze how the variation of feature $ x(\nu_j)$ affect $\mathcal{L}_{ce}(\nu_i)$ based on Eq.~\ref{Eq:ce_spec_grad}. For simplicity, we consider the case $\nu_i \neq 0$ and $\nu_i = 0$ separately. 
For $\nu_i \neq 0$ and $ \nu_j \neq 0$, the SpG of $\mathcal{L}_{ce}$ becomes 
\begin{equation}
    \begin{aligned}
        \frac{{\partial \mathcal{L}_{ce}(\nu_i)}}{{\partial x(\nu_j)}} 
         \simeq & -\sum _c \frac{1}{2}k(\nu_j,c) \delta_{\nu_j}(\nu_i)b(-\nu_i,c), 
    \end{aligned}
\label{Eq:ce_spec_grad_ineq0}
\end{equation}
which indicates that $\mathcal{L}_{ce}(\nu_i)$ is affected only by the CNN feature $x(\nu_j)$ with same frequency. 
For $\nu_i = 0$, the gradient consists of an additional term $\sum _c \frac{1}{2}k(\nu_j,c)D_{0}(\nu_i) s(-\nu_j, c)$. 
Recalls that segmentation map $S(t, c)$ is usually predicted upon the LRG as mentioned in section \ref{sec:intro}. As a result, $s(\nu, c)$ and $s(-\nu_j, c)$ should therefore be small when $\nu_j$ is large. Hence, we have 
\begin{equation}
    \begin{aligned}
        \frac{{\partial \mathcal{L}_{ce}(\nu_i)}}{{\partial x(\nu_j)}} 
         \ll 1, \text{if~} \nu_j \text{~is large} \text{~and~} \nu_i = 0. 
    \end{aligned}
\label{Eq:ce_spec_grad_ieq0}
\end{equation}
It follows from Eq.~\ref{Eq:ce_spec_grad_ineq0} and Eq.~\ref{Eq:ce_spec_grad_ieq0} that $\mathcal{L}_{ce}(\nu_i)$ is affected only by the CNN feature $x(\nu_j)$ with near frequency. 
Thus, removing the features $x(\nu_j)$ at high-frequency $\nu_j$ does not effect $\mathcal{L}_{ce}(\nu_i)$ at low-frequency $\nu_i$. \textcolor{green}{This is our fourth key observation.} 
We provide the numerical validation of third and fourth observations in section~\ref{ssec:exp_spec_ce_grad:grad}.
With these observations, it becomes possible to keep the performance of SSNNs while reducing the feature size, as well as the high-frequency components, of the decoder in SSNNs. 


\section{Validation and Applications}\label{sec:experiments}
This section aims to validate the spectral analysis in section~\ref{sec:method} and further propose some applications such as the feature truncation and block-wise annotation. 
In section~\ref{ssec:exp_spec_ce_grad}, we show that  
(i) the frequency components of CE, $\mathcal{L}_{ce}(\nu)$, is dominated by the low-frequency component as mentioned in section~\ref{sec:method} and (ii) the SpG,  $\frac{{\partial \mathcal{L}_{ce}(\nu_i)}}{{\partial x(\nu_j)}}$, can be well approximated as the delta function as shown in Eq.~\ref{Eq:ce_spec_grad_ineq0}.
Based on these numeric validations, we identify the efficient LRGs and apply the grids onto the features in CNNs and the groundtruth annotation. This leads us to two applications : (1) Feature truncation and (2) Block-wise annotation that down-sample the segmentation maps into LRGs. Experiments for these applications are  detailed in section~\ref{ssec:exp_feat_trunc} and section~\ref{ssec:exp_blk_annot}, respectively. 
\\
\noindent\textbf{Datasets and Segmentation Networks.} 
We examine the experiments upon the following three semantic segmentation datasets: PASCAL semantic segmentation benchmark \cite{everingham2015pascal}, DeepGlobe land-cover classification challenge \cite{demir2018deepglobe} and Cityscapes pixel-level semantic labeling task \cite{cordts2016cityscapes} (denoted as PASCAL, DeepGlobe and Cityscapes respectively). For segmentation networks, we utilize DeepLab v3+~\cite{chen2018encoder} and Deep Aggregation Net (Dan)~\cite{kuo2018dan}. Implementation details of experiments can be seen in  appendix~\ref{ssec:implement}. 

\subsection{Validation of Spectral Analysis}\label{ssec:exp_spec_ce_grad}

\subsubsection{Spectral Decomposition of CE.}\label{ssec:exp_spec_ce_grad:ce}
This section aims to demonstrate that CE is mainly contributed by the low-frequency component, as discussed in section~\ref{ssec:method_spectral_ce}, and investigate the efficacy of using the LRG for prediction by the modern networks. 
Following Eq.~\ref{Eq:spec_ce_component}, we define the truncated CE as the frequency components of CE filtered by LRG
\begin{equation}
    \begin{aligned}
        \widehat{\mathcal{L}}_{ce}({\nu}_{max}) = { \sum _{\nu}^{{\nu}_{max}}} \mathcal{L}_{ce}(\nu),
    \end{aligned}
    \label{Eq:spectral_ce_bandlimit}
\end{equation}
where ${\nu}_{max}$ is the band limit that is half of the LRG size. Further, we define the relative loss rate of CE 
\begin{equation}
    R_{ce}({\nu}_{max}) = |1 - \frac{\widehat{\mathcal{L}}_{ce}({\nu}_{max})}{\mathcal{L}_{CE}}|
    \label{Eq:r_max}
\end{equation}
as a measurement to the loss rate of $\mathcal{L}_{CE}$ due to the band limit ${\nu}_{max}$ of segmentation maps. Hence, an efficient LRG can be defined when $R_{ce}({\nu}_{max})$ is negligible. 
We evaluate $|b(\nu)|$, $|\widehat{y}(\nu)|$, ${\mathcal{L}}_{ce}(\nu)$ and $R_{ce}({\nu}_{max})$ of DeepLab v3+ and Dan on various datasets (PASCAL, DeepGlobe and Cityscapes) where  $|b(\nu)|$ and $|\widehat{y}(\nu)|$ are the averaged power spectrum of $b(\nu,c)$ and $\widehat{y}(\nu,c)$ over all semantic classes $c$ respectively. 
The results are shown in the first three rows of Fig.~\ref{Fig:spectral_model_stat}.~\footnote{The profiles of $|b(\nu)|$ and $|\widehat{y}(\nu)|$ are normalized with respect to their corresponding maximal values for better comparison over datasets.} To monitor the training progress, the evaluation of ${\mathcal{L}}_{ce}(\nu)$ at both the initial and final stages of network training are further shown  in the forth row of the Fig.~\ref{Fig:spectral_model_stat}. Clearly, the first two rows of Fig.~\ref{Fig:spectral_model_stat} indicate that both ground truth annotation $|b(\nu)|$ and prediction $|\widehat{y}(\nu)|$ consist of more low-frequency power. Moreover,  $|\widehat{y}(\nu)|$ is indeed very small in the high-frequency region thus leads to small ${\mathcal{L}}_{ce}(\nu)$, as shown in the third row of the Fig.~\ref{Fig:spectral_model_stat}  
\begin{figure}[h]
    \centering
    \includegraphics[width=1.0\columnwidth]{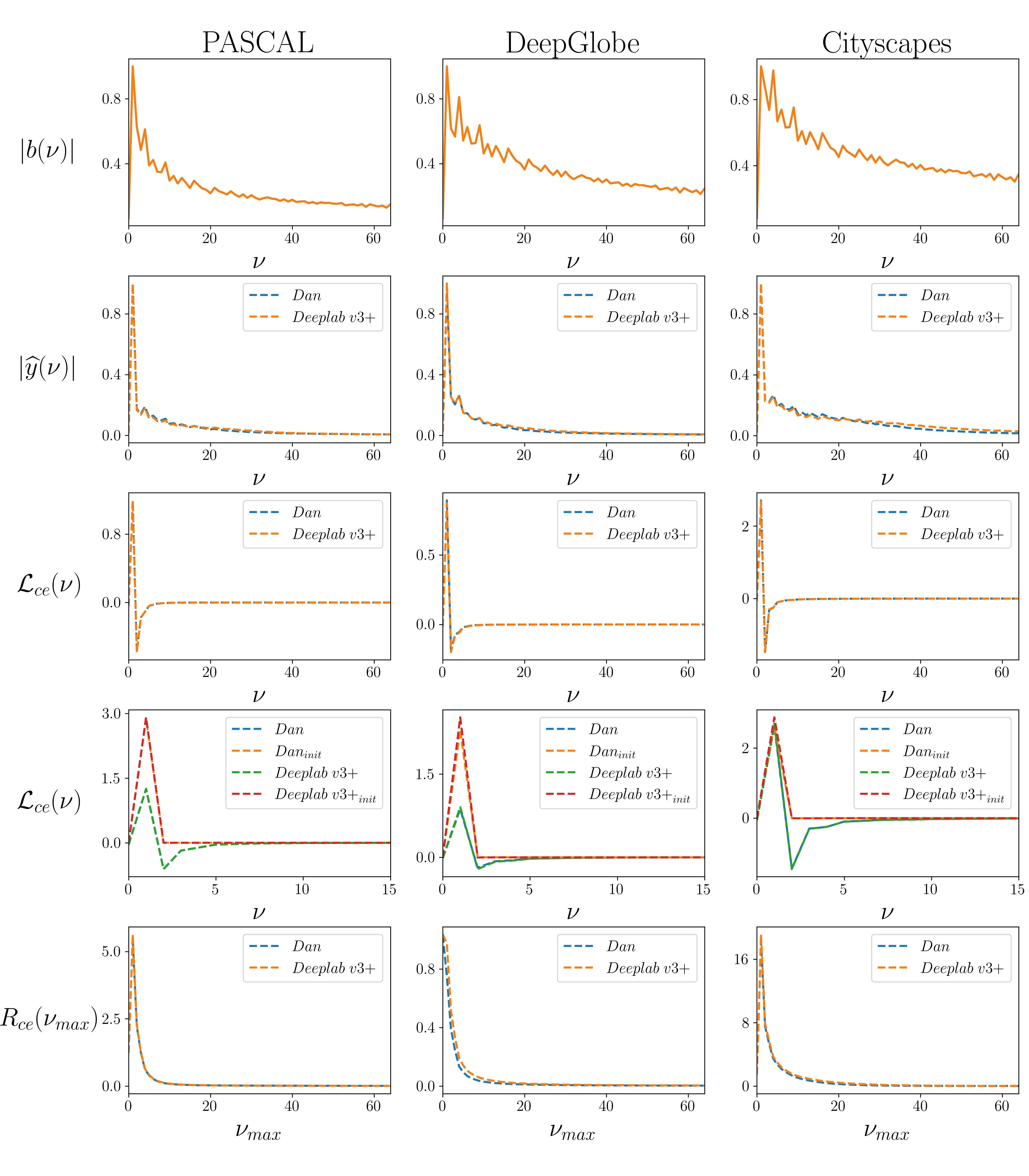} 
    \caption{The magnitude of the spectrum $b(\nu)$, $\widehat{y}(\nu)$, ${\mathcal{L}}_{ce}(\nu)$ and the relative loss rate $R_{ce}({\nu}_{max})$ of CE are evaluated based on DeepLab v3+ and Dan.  The comparison of ${\mathcal{L}}_{ce}(\nu)$ for both initial and final training stages are illustrated in the forth row; the suffix $_{init}$ additionally denotes the profiles of the network at initial stage.
    }
    \label{Fig:spectral_model_stat}
\end{figure}
%
Furthermore, the forth row of the Fig.~\ref{Fig:spectral_model_stat} reveals that the low-frequency components of ${\mathcal{L}}_{ce}(\nu)$  decreases more aggressively in training SSNNs. In other word, the SSNNs learn to capture low-frequency components more effectively.   
These results support our first key observation that CE is mainly contributed by the low-frequency components and the spectral bias indeed exists for the SSNNs as well.
\begin{table}[h]
\centering
\caption{$R_{ce}({\nu}_{max})$ the relative loss rate of CE under various band limit ${{\nu}_{max}}$, evaluated on the PASCAL, DeepGlobe and Cityscapes datasets.}
\label{table:spectral_stat}
\resizebox{1.0\columnwidth}{!}{
\begin{tabular}{ccccccccc}
\multicolumn{1}{c|}{Dataset \textbackslash ${\nu}_{max}$} & \multicolumn{1}{c|}{256} & \multicolumn{1}{c|}{64} & \multicolumn{1}{c|}{48} & \multicolumn{1}{c|}{40} & \multicolumn{1}{c|}{32} & \multicolumn{1}{c|}{24} & \multicolumn{1}{c|}{16} & 8 \\ \hline
\multicolumn{9}{c}{DeepLab v3+} \\ \hline
\multicolumn{1}{c|}{PASCAL} & \multicolumn{1}{c|}{0} & \multicolumn{1}{c|}{0.008} & \multicolumn{1}{c|}{0.011} & \multicolumn{1}{c|}{0.014} & \multicolumn{1}{c|}{\textbf{0.017}} & \multicolumn{1}{c|}{0.023} & \multicolumn{1}{c|}{0.040} & \textbf{0.184} \\
\multicolumn{1}{c|}{DeepGlobe} & \multicolumn{1}{c|}{0} & \multicolumn{1}{c|}{0.003} & \multicolumn{1}{c|}{0.004} & \multicolumn{1}{c|}{0.005} & \multicolumn{1}{c|}{\textbf{0.007}} & \multicolumn{1}{c|}{0.009} & \multicolumn{1}{c|}{0.017} & \textbf{0.053} \\
\multicolumn{1}{c|}{Cityscape} & \multicolumn{1}{c|}{0} & \multicolumn{1}{c|}{0.014} & \multicolumn{1}{c|}{0.020} & \multicolumn{1}{c|}{0.030} & \multicolumn{1}{c|}{\textbf{0.065}} & \multicolumn{1}{c|}{0.175} & \multicolumn{1}{c|}{0.517} & \textbf{1.782} \\ \hline
\multicolumn{9}{c}{Dan} \\ \hline
\multicolumn{1}{c|}{PASCAL} & \multicolumn{1}{c|}{0} & \multicolumn{1}{c|}{0.008} & \multicolumn{1}{c|}{0.011} & \multicolumn{1}{c|}{0.014} & \multicolumn{1}{c|}{\textbf{0.017}} & \multicolumn{1}{c|}{0.023} & \multicolumn{1}{c|}{0.045} & \textbf{0.197} \\
\multicolumn{1}{c|}{DeepGlobe} & \multicolumn{1}{c|}{0} & \multicolumn{1}{c|}{0.004} & \multicolumn{1}{c|}{0.006} & \multicolumn{1}{c|}{0.007} & \multicolumn{1}{c|}{\textbf{0.010}} & \multicolumn{1}{c|}{0.014} & \multicolumn{1}{c|}{0.027} & \textbf{0.084} \\
\multicolumn{1}{c|}{Cityscape} & \multicolumn{1}{c|}{0} & \multicolumn{1}{c|}{0.017} & \multicolumn{1}{c|}{0.040} & \multicolumn{1}{c|}{0.075} & \multicolumn{1}{c|}{\textbf{0.153}} & \multicolumn{1}{c|}{0.328} & \multicolumn{1}{c|}{0.723} & \textbf{2.004}
\end{tabular}
}
\end{table}

Next, to investigate the efficacy of LRGs, we plot the $R_{ce}({\nu}_{max})$ with respect to the change of the band limit ${\nu}_{max}$ in the last row of Fig.~\ref{Fig:spectral_model_stat}. 
For comparison, we examine the numeric value of $R_{ce}({\nu}_{max})$ in Table~\ref{table:spectral_stat}. Note that the band limit ${\nu}_{max}$ of our baseline grid is 256 (i.e. image size 513). We have $\mathcal{L}_{CE}=\widehat{\mathcal{L}}_{ce}(256)$. Obviously, one can see that the loss rate of $\mathcal{L}_{ce}(\nu)$ changes rapidly when $\nu > 8$ and decreases by about a factor of 10 at the band limit ${\nu}_{max}=32$. These empirical evidences suggest that the LRGs with $ 16 \leq {\nu}_{max} \leq 32$ could efficiently sample the segmentation map without significant information loss. The efficiency of these LRGs is further demonstrated in our proposed applications, \ie feature truncation and block-wise annotation, in sections~\ref{ssec:exp_feat_trunc} and section~\ref{ssec:exp_blk_annot}, respectively. We shall expect a better efficacy of the feature truncation and the block-wise annotation on the PASCAL and DeepGlobe datasets since $R_{ce}({\nu}_{max})$ on the Cityscapes dataset is significantly larger than those on the PASCAL and DeepGlobe datasets.  

\subsubsection{Spectral Gradient.}\label{ssec:exp_spec_ce_grad:grad}
We now turn to investigate the SpGs introduced in section~\ref{ssec:method_spec_grad}, including $\frac{\partial y(\nu_i)}{\partial x(\nu_j)}$ in Eq.~\ref{Eq:convlayer_spec_grad} and $\frac{{\partial {\mathcal{L}}_{ce}(\nu_i)}}{{\partial x(\nu_j)}}$ in Eq.~\ref{Eq:ce_spec_grad} for DeepLab v3+ and Dan. The SpGs $\frac{{\partial {\mathcal{L}}_{ce}(\nu_i)}}{{\partial x(\nu_j)}}$ are evaluated on three datasets (PASCAL, DeepGlobe, and Cityscapes datasets), where $x(\nu_j)$ is the spectra of ASPP (atrous spatial pyramid pooling) features. The images of these SpGs are shown in Fig.~\ref{Fig:spec_grad_deeplab} and Fig.~\ref{Fig:spec_grad_dan}. 
Moreover, we also validate the SpGs $\frac{\partial y(\nu_i)}{\partial x(\nu_j)}$ for operations including the convolution, ReLU, and bilinear up-sampling within the decoder module of these networks; where $x(\nu_j)$ and $y(\nu_i)$ are the input spectra and output spectra of the operations, respectively. The images of the SpGs are shown in Fig.~\ref{Fig:spec_grad}. Only the SpGs of the frequency $\nu_i=(0, \frac{M}{4}, \frac{M}{2})$ are shown as examples in  Fig.~\ref{Fig:spec_grad_deeplab}, Fig.~\ref{Fig:spec_grad_dan} and Fig.~\ref{Fig:spec_grad} where $M$ is the size of input features and $M=33$ in this experiment.

\begin{figure}[ht]
    \captionsetup[subfigure]{}
    \centering
    \begin{subfigure}[]{0.45\columnwidth}
        \centering
        \includegraphics[width=1.0\columnwidth]{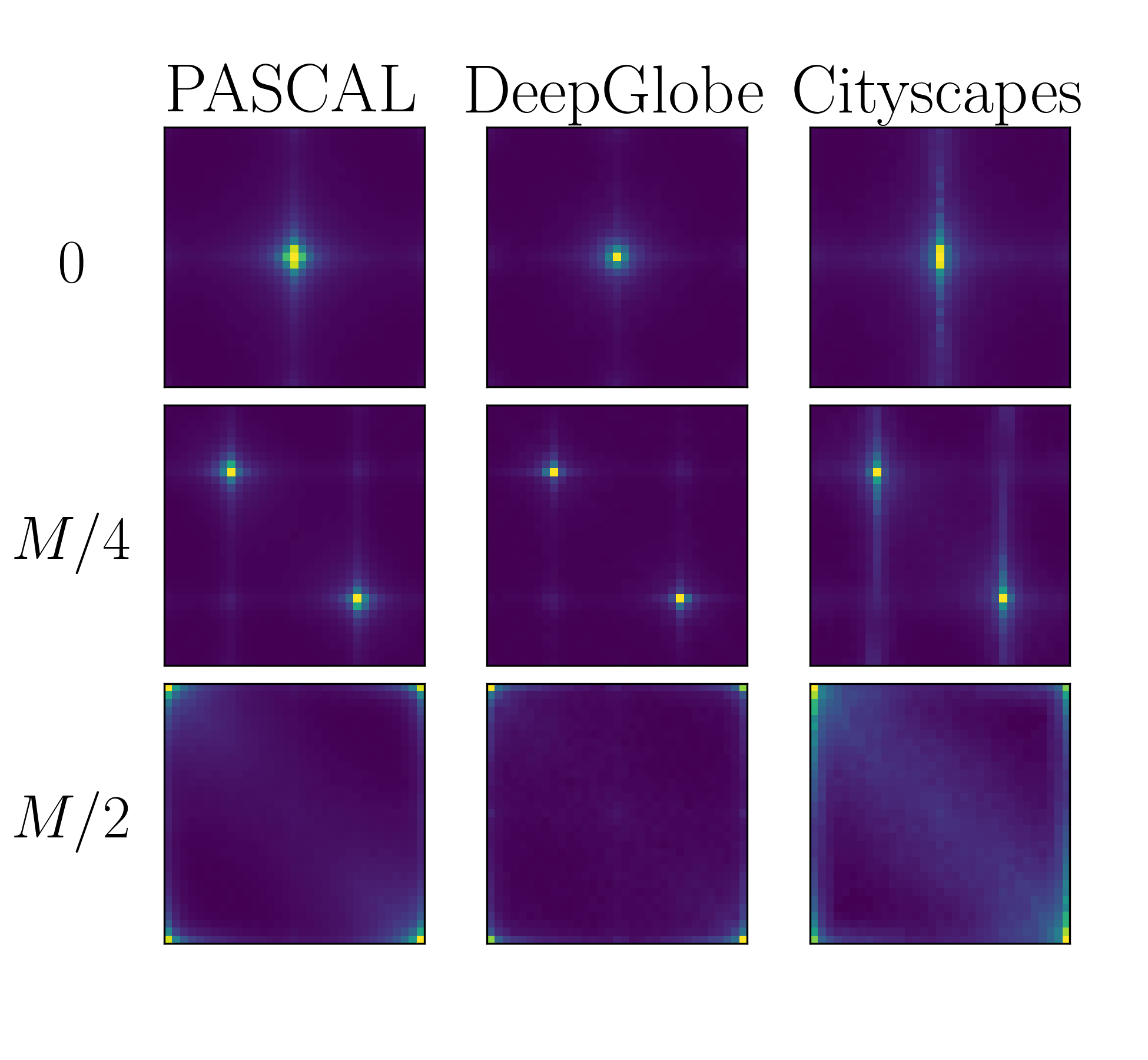} 
        \caption{DeepLab v3+}
        \label{Fig:spec_grad_deeplab}
    \end{subfigure}
    \begin{subfigure}[]{0.45\columnwidth}
        \centering
        \includegraphics[width=1.0\columnwidth]{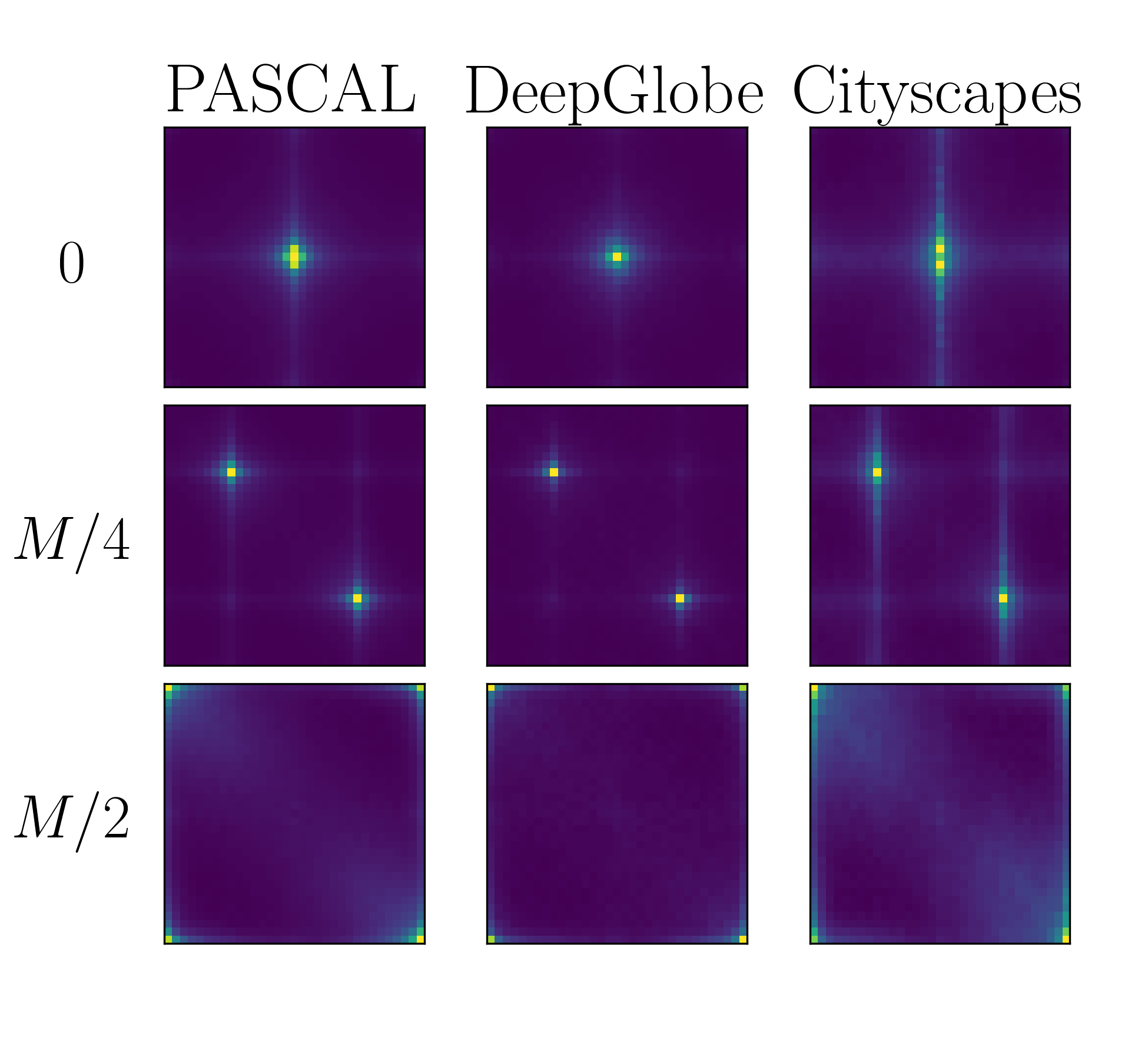}
        \caption{Dan}
        \label{Fig:spec_grad_dan}    
    \end{subfigure}
    \caption{The SpG ${{\partial {\mathcal{L}}_{ce}(\nu_i)}}/{{\partial x(\nu_j)}}$ of frequency $\nu_i=0$, $\frac{M}{2}$ and $\frac{M}{2}$ for DeepLab v3+ and Dan evaluated on PASCAL, DeepGlobe and Cityscapes datasets.}
\end{figure}
\begin{figure}[ht]
    \centering
    \includegraphics[width=0.45\columnwidth]{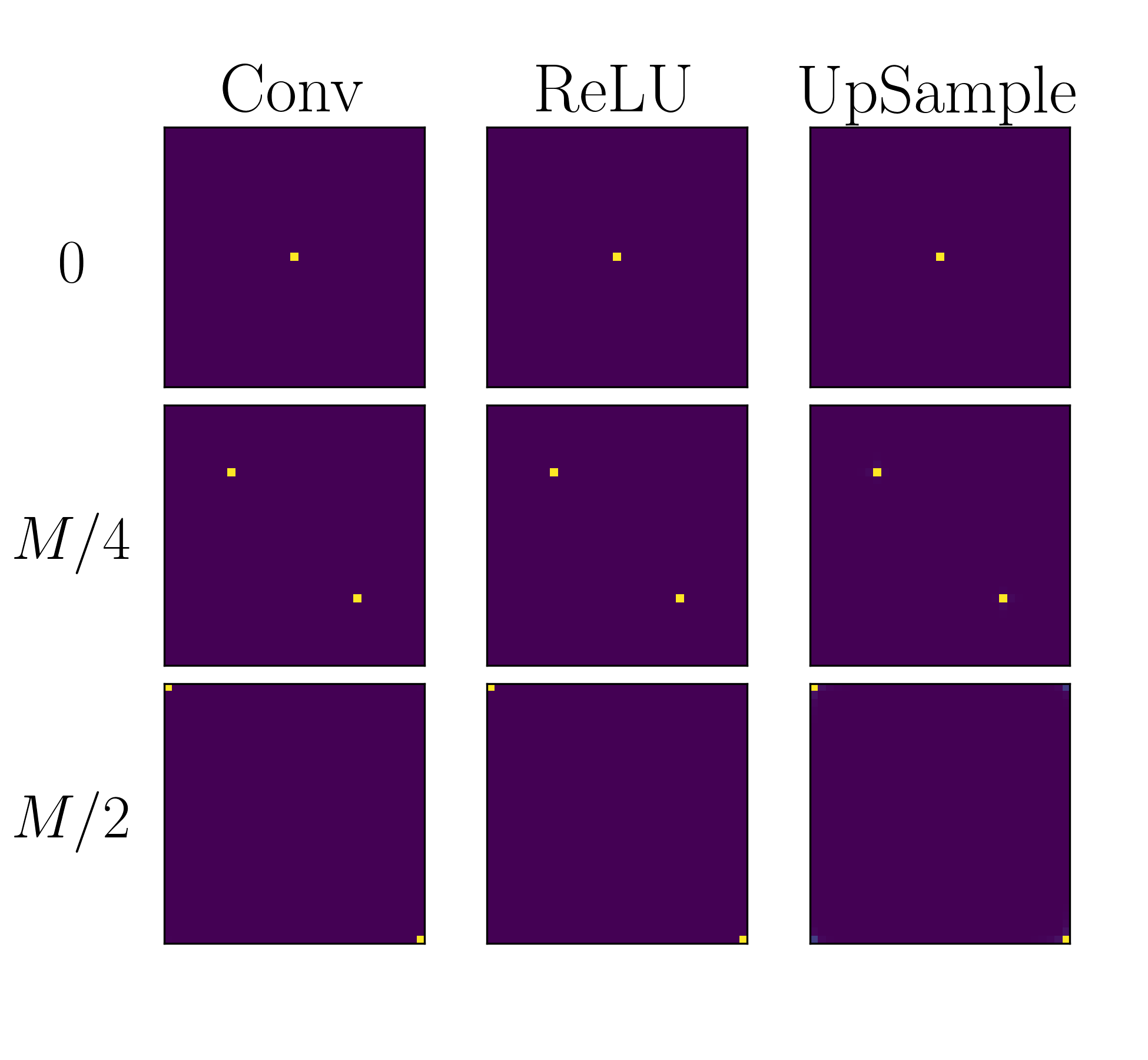}
    \caption{The SpGs ${{\partial y(\nu_i)}}/{{\partial x(\nu_j)}}$ of frequency $\nu_i=0$, $\frac{M}{2}$ and $\frac{M}{2}$ for the convolution, ReLU, and bilinear up-sampling operations denoted as Conv, ReLU, Upsample, respectively.}
    \label{Fig:spec_grad}
\end{figure}

Obviously, one can see that the SpG $\frac{{\partial y(\nu_i)}}{{\partial x(\nu_j)}}$ of the convolution, ReLU and bilinear up-sampling is a delta function, for all $\nu_i$, from Fig.~\ref{Fig:spec_grad}. These results are consistent with our discussion for Eq.~\ref{Eq:convlayer_spec_grad}. 
Moreover, clearly from Fig.~\ref{Fig:spec_grad_deeplab} and Fig.~\ref{Fig:spec_grad_dan}, the SpG  $\frac{{\partial {\mathcal{L}}_{ce}(\nu_i)}}{{\partial x(\nu_j)}}$ that accumulating all the operations in the decoder module can also be well approximated by the delta function. These results verify the approximation of Eq.~\ref{Eq:ce_spec_grad} and agrees with our discussion in section~\ref{ssec:method_spec_grad} that the feature component at frequency $\nu_i$ only affects ${\mathcal{L}}_{ce}$ at the frequency near $\nu_i$. 
Together with observations that the SSNNs capture the low-frequency features more effectively and $\mathcal{L}_{ce}$ is mainly contributed by low frequencies, the above results inspire us to further consider the removal of redundant high-frequency features in section~\ref{ssec:exp_feat_trunc}.

\subsection{Application on Feature Truncation}\label{ssec:exp_feat_trunc}
We propose feature truncation as a model reduction method that reduces the cost of SSNNs without degrading the performance. Our feature truncation method is intuitively performed by removing the high-frequency components of the decoder features $x(\nu)$. 
Furthermore, from the discussion of SpG in section~\ref{ssec:exp_spec_ce_grad:grad}, we now can directly apply the analysis of $R_{ce}({\nu}_{max})$ in section~\ref{ssec:exp_spec_ce_grad:ce} to estimate the LRG not only for segmentation map but also for the decoder feature. In the section, we hereafter refer to the feature size for the decoder feature as the \textbf{LRG size}.
We further adopt the Soft Filter Punning (SFP) method~\cite{he2019asymptotic} and combine it with our feature truncation. The experiments are done based on the following conditions: For SFP, we set pruning rates (PR) to 20\%, 40\%, and 60\% for the encoder and 20\% for the decoder, since the number of parameters in the encoder is much larger than that in the decoder and is potentially over-parameterized. On the other hand, we apply 
feature truncation on the decoder by down-sampling the decoder features from the original prediction grid with size 129 to the efficient LRG sizes from 97 to 17.
Note that the LRG size with $ 32 \geq {\nu}_{max} \geq 16$ (i.e. LRG size from 65 to 33) are expected to reach optimal efficacy of feature truncation as suggested in section~\ref{ssec:exp_spec_ce_grad}. 
In our experiments, all the down-sampling are done via bilinear interpolation. 

The feature truncation effectively reduce the FLOPs for models in our experiments. The FLOPs of DeepLab v3+ and Dan for various pruning rates and feature sizes can be seen in Table~\ref{table:feat_trunc_flops:flops}.
Obviously from the table, the FLOPs decrease as the feature size decreases from 129 to 17 or as PR increases from baseline to 60\%, for both DeepLab v3+ and Dan. 
Note that DeepLab v3+ has much more parameters than Dan; the number parameters of the decoder are 1.3 and 0.4 million for DeepLab v3+ and Dan, respectively. Calculating from Table~\ref{table:feat_trunc_flops:flops}, DeepLab v3+ requires 29\% to 60\% more flops than Dan without feature truncation when the PR grows from baseline to 60\%. However, the flops converge almost to the same values for both models when features are truncated at size 33, indicating a general yet effective cost reduction of the decoder. 
Combining the traditional pruning method (SFP) and feature truncation, our approach effectively reduces the cost of both the encoder and the decoder of SSNNs.

\begin{table}[h]
    \centering
    \resizebox{0.85\columnwidth}{!}{
    \begin{tabular}{cccccccc}
        \hline \multicolumn{8}{c}{DeepLab v3+} \\ \hline
        SFP setups \textbackslash LRG size & 129 & 97 & 81 & 65 & 49 & 33 & 17 \\ \hline
        \multicolumn{1}{c|}{Baseline} & 139 & 120 & 113 & 107 & 102 & \textbf{98} & 96 \\
        \multicolumn{1}{c|}{20\% PR for encoder} & 100 & 86 & 81 & 76 & 73 & \textbf{70} & 69 \\
        \multicolumn{1}{c|}{40\% PR for encoder} & 77 & 63 & 58 & 53 & 50 & \textbf{47} & 46 \\
        \multicolumn{1}{c|}{60\% PR for encoder} & 59 & 45 & 40 & 35 & 32 & \textbf{29} & 28 \\ \hline
        \multicolumn{8}{c}{Dan} \\ \hline
        SFP setups \textbackslash LRG size & 129 & 97 & 81 & 65 & 49 & 33 & 17 \\ \hline
        \multicolumn{1}{c|}{Baseline} & 107 & 102 & 100 & 98 & 97 & \textbf{96} & 96 \\
        \multicolumn{1}{c|}{20\% PR for encoder} & 78 & 73 & 72 & 70 & 69 & \textbf{69} & 68 \\
        \multicolumn{1}{c|}{40\% PR for encoder} & 55 & 50 & 49 & 47 & 46 & \textbf{46} & 45 \\
        \multicolumn{1}{c|}{60\% PR for encoder} & 37 & 32 & 31 & 29 & 28 & \textbf{28} & 27
    \end{tabular}
    }
    \caption{FLOPs for the models with SFP and feature truncation; denoted in unit of $10^9$ FLOPs}
    \label{table:feat_trunc_flops:flops}
\end{table}

\begin{figure}[h]
\captionsetup[subfigure]{}
\centering
\begin{subfigure}[]{1.0\columnwidth} 
    \centering
    \includegraphics[width=1.\columnwidth]{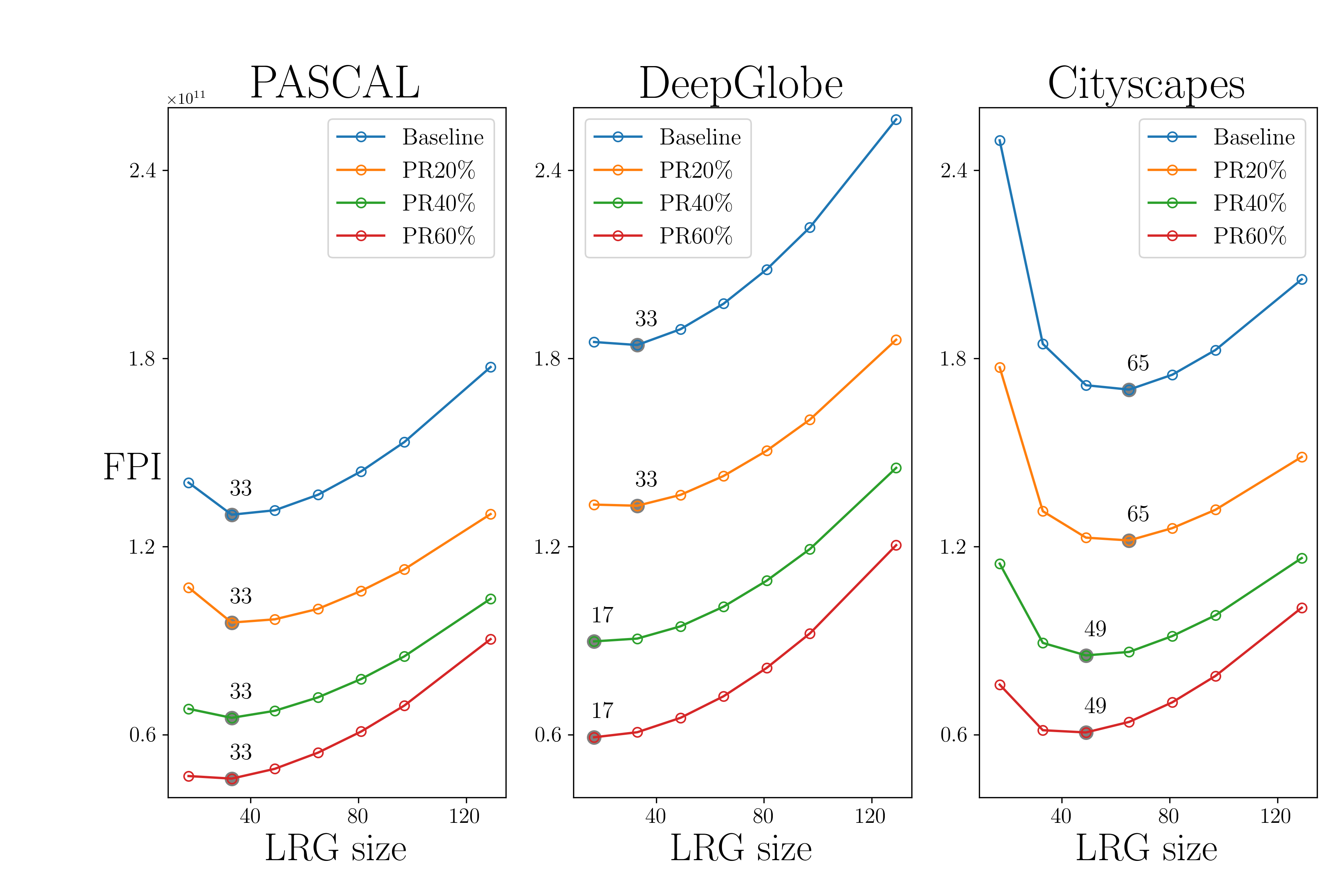}  
    \caption{DeepLab v3+}
    \label{Fig:flops_per_iou_deeplab}
\end{subfigure}
\begin{subfigure}[]{1.0\columnwidth} 
    \centering
    \includegraphics[width=1.\columnwidth]{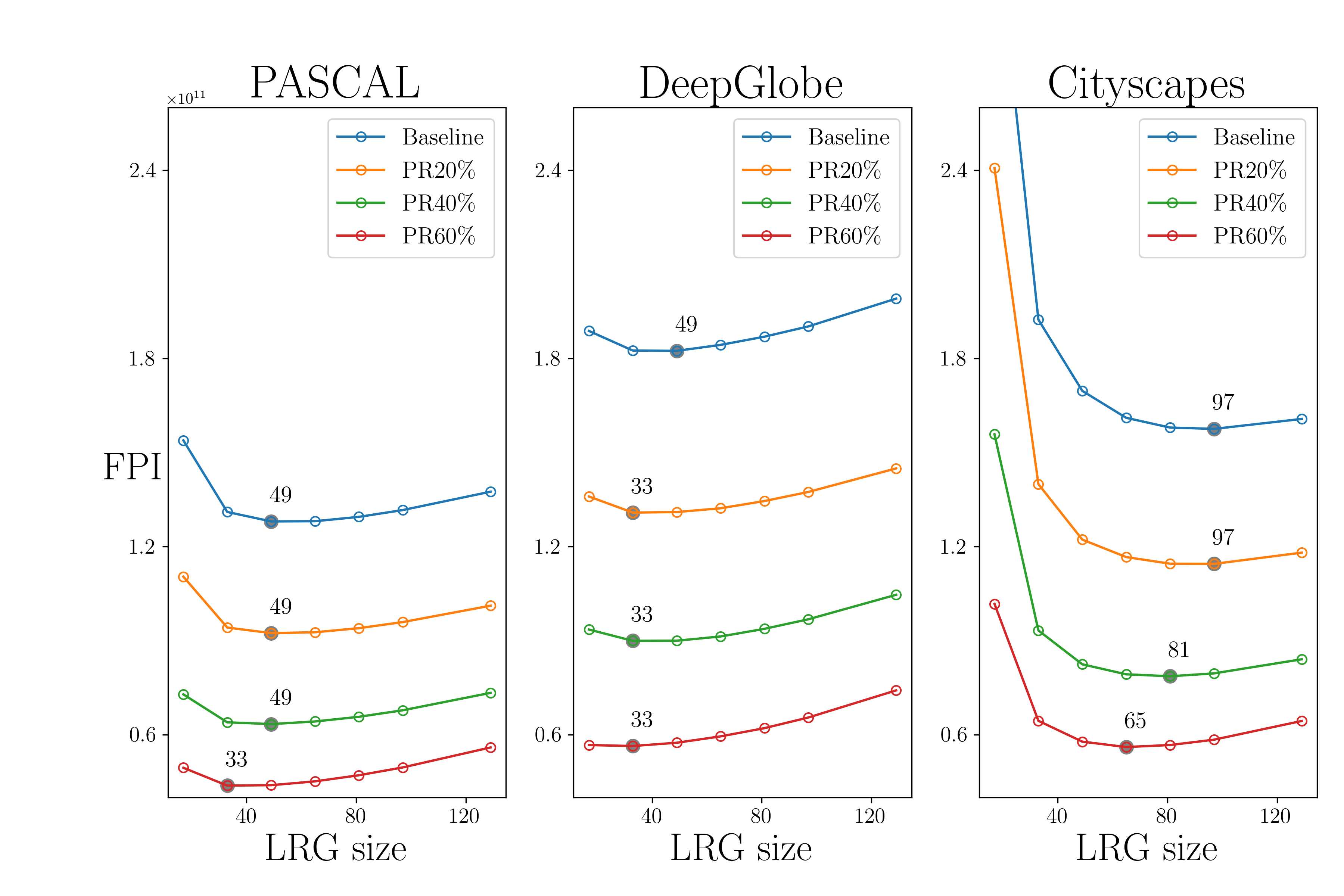}
    \caption{Dan}
    \label{Fig:flops_per_iou_dan}
\end{subfigure}
\caption{ 
The FPI for DeepLab v3+ and Dan with various LRG sizes and prune rates of SFP for the feature truncation. The LRG with minimal FPI is annotated with a closed circle for each line.}
\label{Fig:flops_per_iou}
\end{figure}

\begin{figure}[h]
    \begin{subfigure}[]{1.0\columnwidth}
        \centering
        \includegraphics[width=1.\columnwidth]{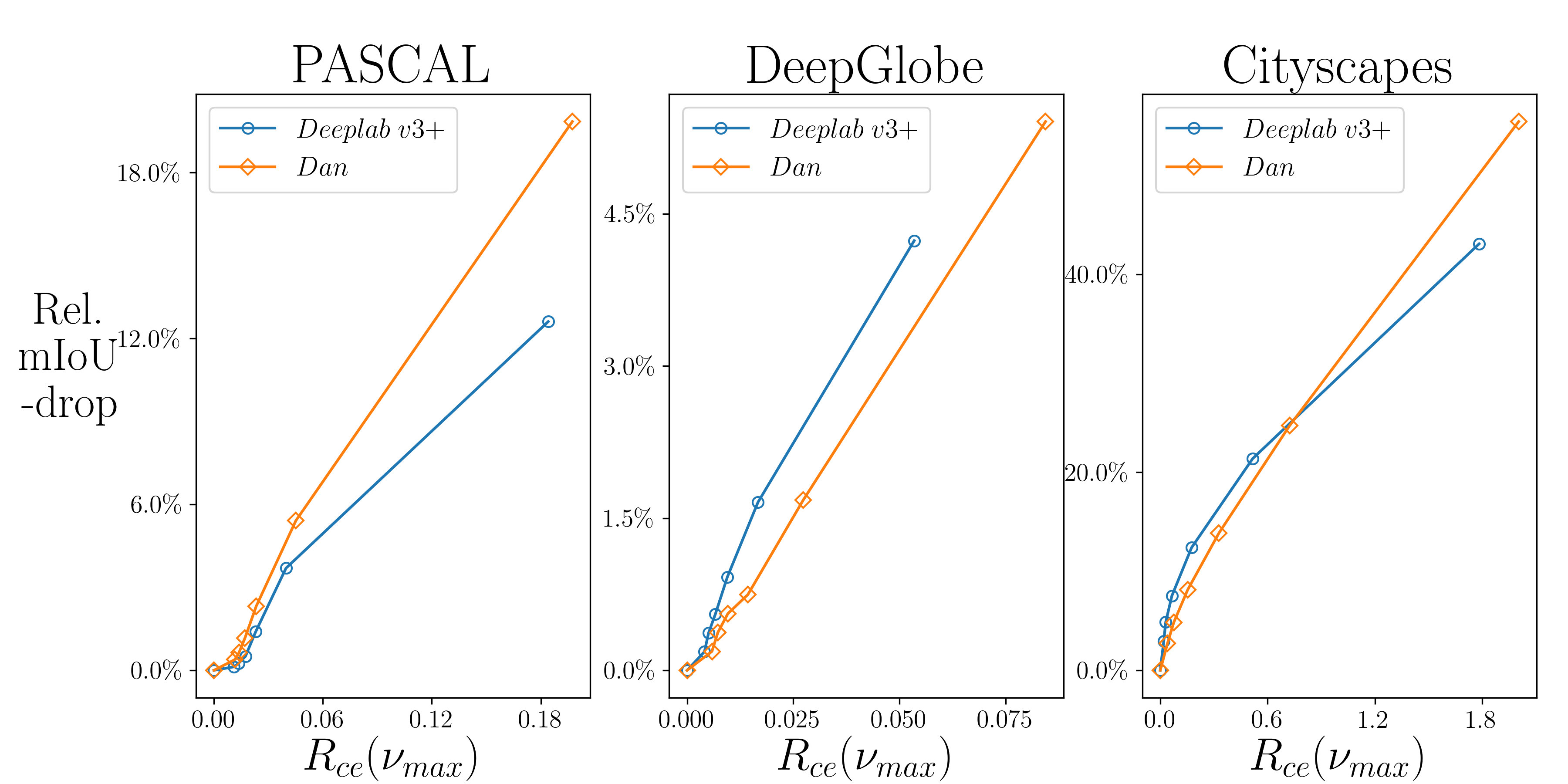} 
        \caption{Result separately for the three datasets. Each line consists of 7 data with LRG sizes 129,97,81,65,49,33, and 17; the larger the LRG size, the lower the $R_{ce}({\nu}_{max})$.}
        \label{Fig:feature_truncation_iou_R:sep}
    \end{subfigure}
    \begin{subfigure}[]{1.0\columnwidth}
        \centering
        \includegraphics[width=1.\columnwidth]{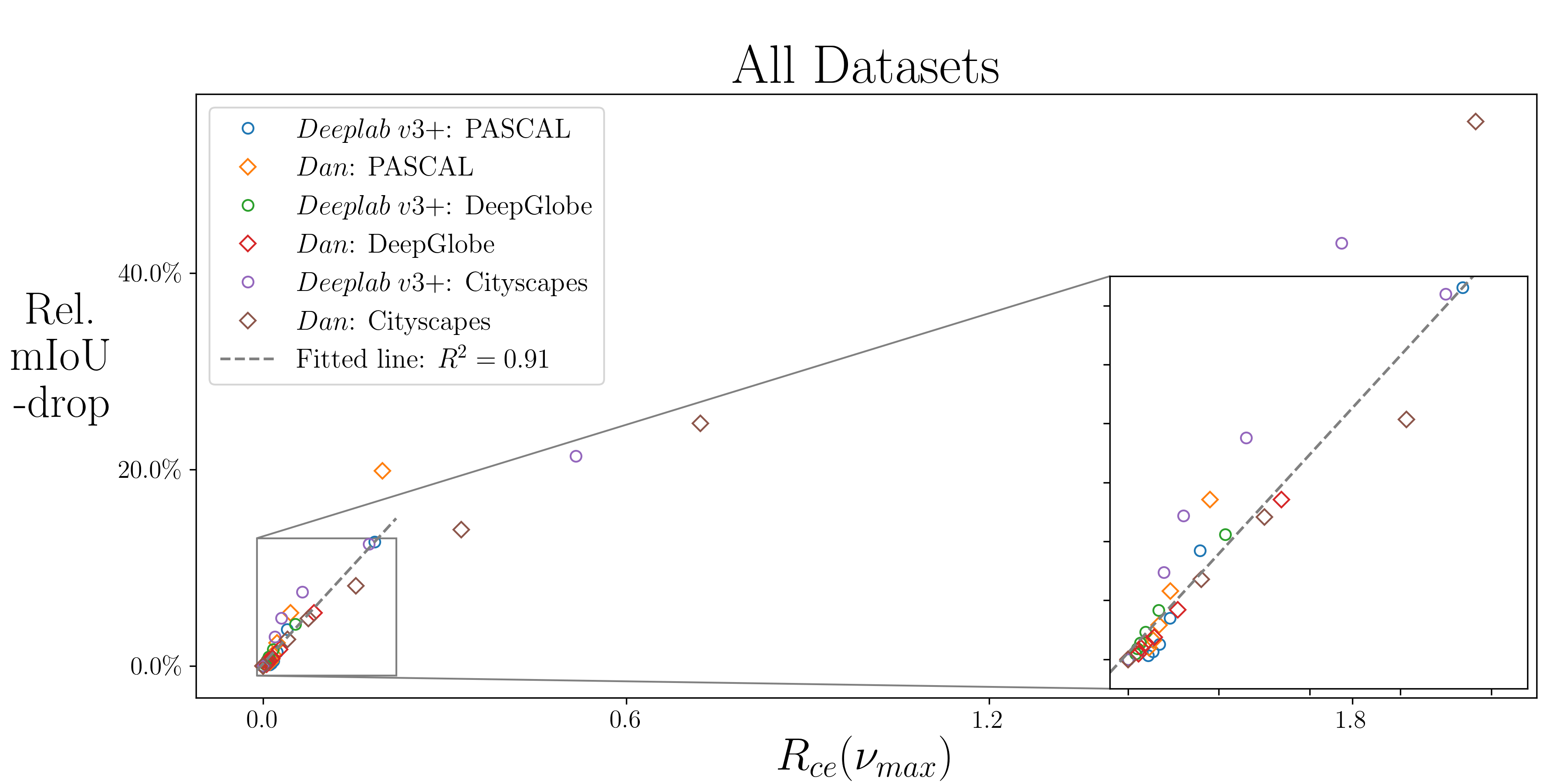}
        \caption{Result jointly for the three datasets with fitted line at zoom-in region along with the $R^2$ value. The fitted line is $0.6725R_{ce}({\nu}_{max})+ 0.0022$}
        \label{Fig:feature_truncation_iou_R:joint}
    \end{subfigure}
        \caption{The correlation between relative mIoU-drop and the $R_{ce}({\nu}_{max})$ for the feature truncation. The results are evaluated based on PASCAL, DeepGlobe, and Cityscapes datasets with DeepLab v3+ and Dan models.}
    \label{Fig:feature_truncation_iou_R}
\end{figure}

On the other hand, degradation of the model performance is inevitable when a huge cost reduction is gained by using SFP and feature truncation. 
We evaluate DeepLab v3+ and Dan upon the PASCAL, DeepGlobe, and Cityscapes datasets and relegate the experimental details in appendix~\ref{ssec:feat_detail}, where Fig.~\ref{Fig:iou_flops} illustrates the IoU score of these models with respect to the FLOPs while the numeric details are provided in Table~\ref{table:feat_trunc:deeplab} and Table~\ref{table:feat_trunc:dan}. 
Besides, we define \textbf{FLOPs per IoU score (FPI)} as FLOPs/mIoU to estimate the efficiency of feature truncation for these models that simultaneously includes the cost and performance. The lower the FPI, the better the model's efficiency which minimizes the performance drop while maximizing the cost reduction of inference. The optimal LRG size is thus determined by minimizing FPI. 
The FPI of these models is illustrated in Fig.~\ref{Fig:flops_per_iou}, where the optimal LRG is denoted for each setup. 
From the figure, it can be seen that most of the optimal LRGs lie in the range between 33 and 65, which is consistent with our analysis on $R_{ce}({\nu}_{max})$, except for the optimal LRG sizes with DeepLab model upon the DeepGlobe dataset and those with Dan upon the Cityscapes dataset; the former case has an optimal LRG size 17 while the later case has LRG sizes 81 and 97. Despite being exceptions, these optimal LRGs turn out to match the small $R_{ce}({\nu}_{max})$ in Table~\ref{table:spectral_stat}~\footnote{$R_{ce}({\nu}_{max})$ is 0.053 for LRG size 17 in the former case, and remain small at 0.040 and 0.075 respectively for LRG sizes 81 and 97 in the latter case}. 

Moreover, these optimal LRG sizes seem to be less sensitive to PR, which indicates the performance drop dues to feature truncation is not affected by network pruning. Such a correlation further motivates us to investigate the correlation between the performance drop and the $R_{ce}({\nu}_{max})$, which might help to estimate the optimal LRG all at once without enumerating all possible LRGs and PR. Here we evaluate the relative IoU drop caused by feature truncation in Table~\ref{table:feat_trunc} and illustrate its correlation between $R_{ce}({\nu}_{max})$ in Fig.~\ref{Fig:feature_truncation_iou_R}.
Clearly from Fig.~\ref{Fig:feature_truncation_iou_R:sep}, the positive correlation demonstrates that $R_{ce}({\nu}_{max})$ for all datasets in our experiments. We further prove that this correlation is indeed positive in appendix~\ref{ssec:method_trunc_ioU}. 
Besides, we observe significant mIoU-drop on Cityscapes is well explained by the theoretical estimation $R_{ce}({\nu}_{max})$ in Table~\ref{table:spectral_stat}, where $R_{ce}({\nu}_{max})$ depict the relative loss rate of CE evaluated that positively correlates to relative mIoU-drop as discussed in section~\ref{ssec:method_spectral_ioU}. 
On the other hand, Fig.~\ref{Fig:feature_truncation_iou_R:joint} shows the results upon all datasets jointly and demonstrates a similar positive correlation even for all models and datasets. Furthermore, an apparent linear correlation is observed in the regime with smaller $R_{ce}({\nu}_{max})$, \ie $R_{ce}({\nu}_{max}) < 0.2$. A $R^2$ value of 0.91 and a slope of 0.6725 are obtained by fitting a line for the data in this regime despite that the correlation between IoU and $R_{ce}({\nu}_{max})$ is positive while non-linear. 
The high $R^2$ value indicates that such a non-linear correlation can be well approximated by linear functions thus enabling a simple estimation of IoU-drop by $R_{ce}({\nu}_{max})$.
With a proper threshold of $R_{ce}({\nu}_{max})$, depending on the required model performance, the fitted line in Fig.~\ref{Fig:feature_truncation_iou_R} becomes an efficient estimator for model performance in practice.

In summary, we conclude that the integration of feature truncation and SFP can efficiently reduce the computational cost when $R_{ce}({\nu}_{max})$ is small as demonstrated on the PASCAL, DeepGlobe, and Cityscapes datasets. 
For all setups of SFP, the optimal cost reduction of SSNNs can be achieved by the feature truncation with the optimal LRG size, which is determined by minimizing the proposed FPI that measures simultaneously the cost reduction and the performance drop of SSNNs. 
Moreover, we demonstrate that the performance drop of SSNNs positively correlates to $R_{ce}({\nu}_{max})$ as clearly shown in Fig.~\ref{Fig:feature_truncation_iou_R}. 
Hence, one can easily determine the optimal LRG sizes without evaluating the performance of all possible LRG sizes by combining the FLOPs computed in Table~\ref{table:feat_trunc_flops:flops}, which is independent of datasets and thus can be pre-computed, and the correlation in Fig.~\ref{Fig:feature_truncation_iou_R}, which is general to all three datasets and might generalize well to real-world datasets. 
As mentioned in section~\ref{sec:intro}, SSNNs predict the segmentation maps upon the LRG to save computational costs. Our framework serves as an analysis tool in estimating the efficient LRG size as well as the effective feature size in decoders in saving the cost. 
This further allows SSNNs to dynamically adapt the LRG size for various domains and can be generalized to arbitrary features in CNNs.

\subsection{Application on Block-wise annotation}\label{ssec:exp_blk_annot}
In section~\ref{ssec:exp_spec_ce_grad}, we determine the efficient LRG for the segmentation maps via analyzing $R_{ce}({\nu}_{max})$. In this section, we apply these LRGs to the groundtruth annotations. The resulting block-wise annotation can be considered a weak annotation. We demonstrate that the performance of the SSNNs trained with these block-wise annotations can also be estimated by $R_{ce}({\nu}_{max})$.
We perform the experiment that trains DeepLab v3+ and Dan with the block-wise annotation at various band limit ${\nu}_{max}$ (from 256 to 8) and evaluates the $R_{ce}({\nu}_{max})$ based on the original pixel-wise annotation. Examples of the block-wise annotation and the prediction of the two models upon the PASCAL, DeepGlobe, and Cityscapes datasets are illustrated in Fig.~\ref{Fig:block_annotation}. Note that the block-wise annotation at ${\nu}_{max}=256$ is actually equivalent to the original pixel-wise groundtruth. The experimental results are summarized in Table~\ref{table:blk_annot}. For each ${\nu}_{max}$, we evaluate mIoU score and mIoU drop. Particularly, mIoU-drop is the reduction rate of mIoU with respect to the one band limit ${\nu}_{max}=256$; this drop actually corresponds to the decrements of IoU score caused by the LRG on the annotation.
\begin{figure}[h]
    \centering
    \begin{subfigure}[b]{.9\columnwidth}
        \centering
        \resizebox{1.0\columnwidth}{!}{
        \begin{tabular}[b]{@{}c@{}@{}c@{}@{}c@{}@{}c@{}@{}c@{}@{}c@{}}
            \hline
             ${\nu}_{max}$ & 256 & 32 & 16 & 8 \\
             \hline
            \multicolumn{5}{c}{PASCAL} \\ \hline  \includegraphics[width=.15\columnwidth]{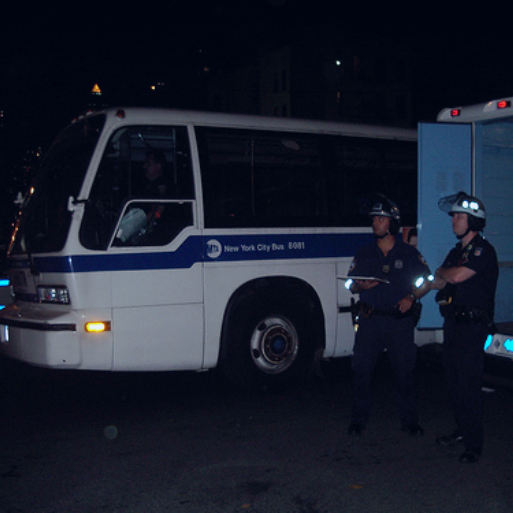} &
            \includegraphics[width=.15\columnwidth]{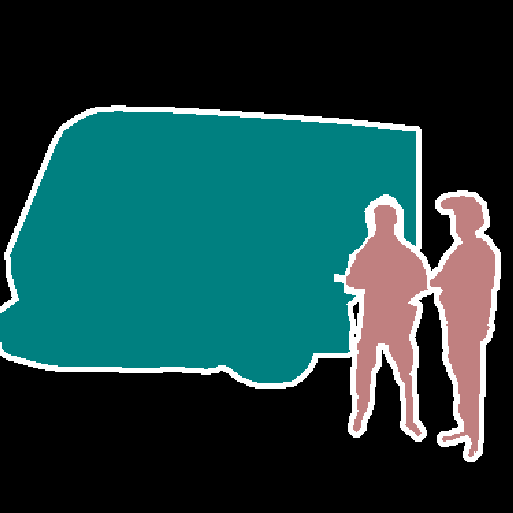} &
            \includegraphics[width=.15\columnwidth]{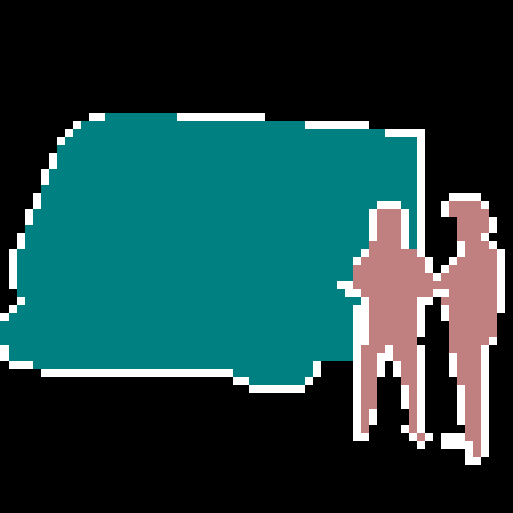} &
            \includegraphics[width=.15\columnwidth]{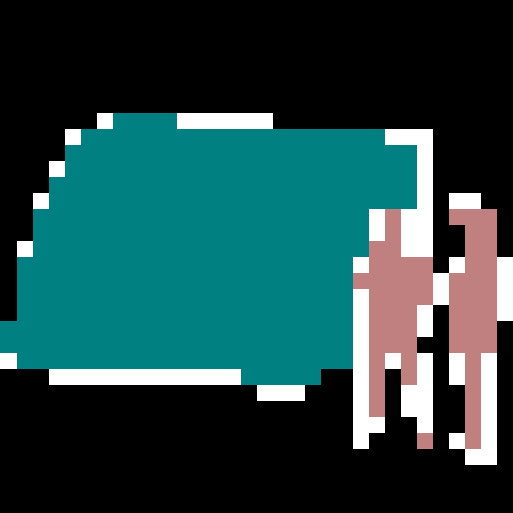} &
            \includegraphics[width=.15\columnwidth]{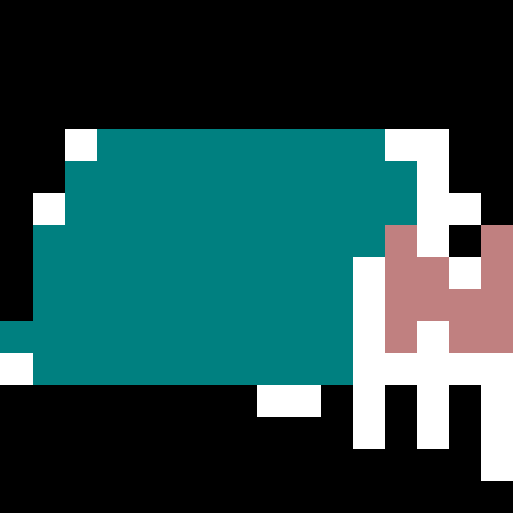} \\
            \centered{DeepLab v3+} &
            \centered{\includegraphics[width=.15\columnwidth]{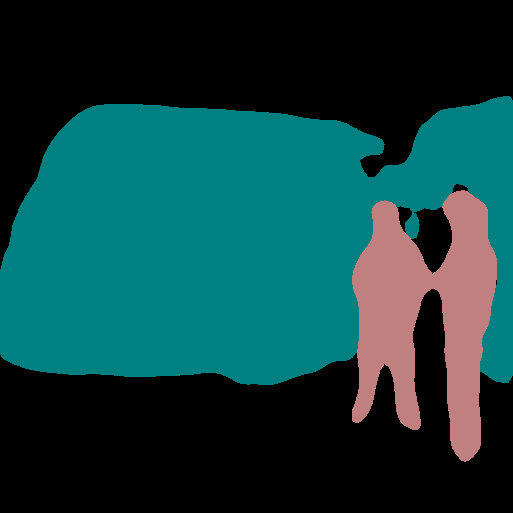}} &
            \centered{\includegraphics[width=.15\columnwidth]{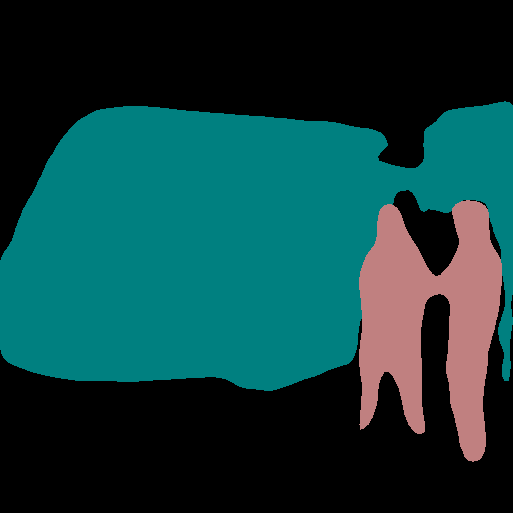}} &
            \centered{\includegraphics[width=.15\columnwidth]{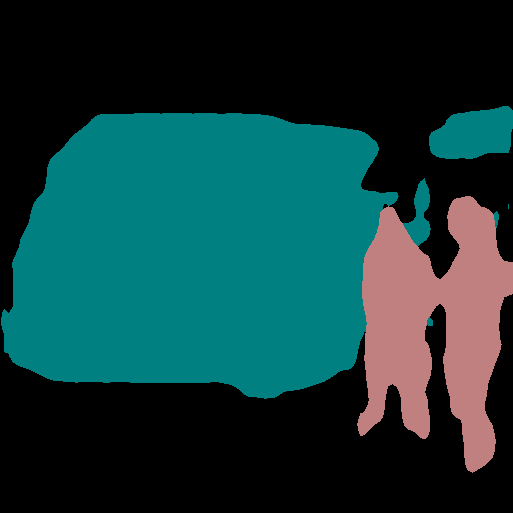}} &
            \centered{\includegraphics[width=.15\columnwidth]{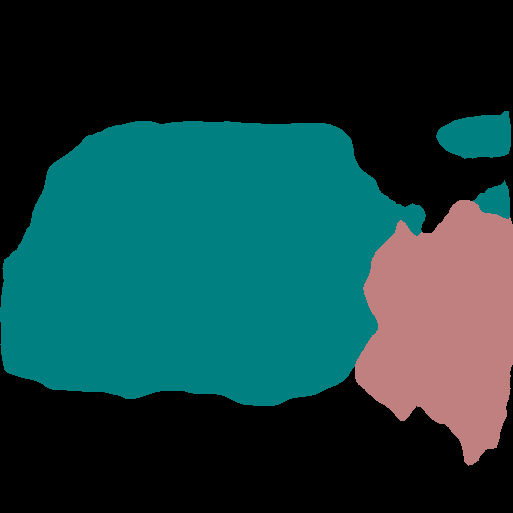}} \\
            \centered{Dan} &
            \centered{\includegraphics[width=.15\columnwidth]{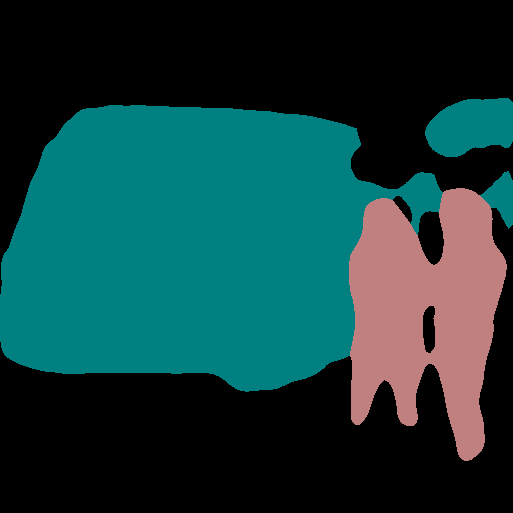}} &
            \centered{\includegraphics[width=.15\columnwidth]{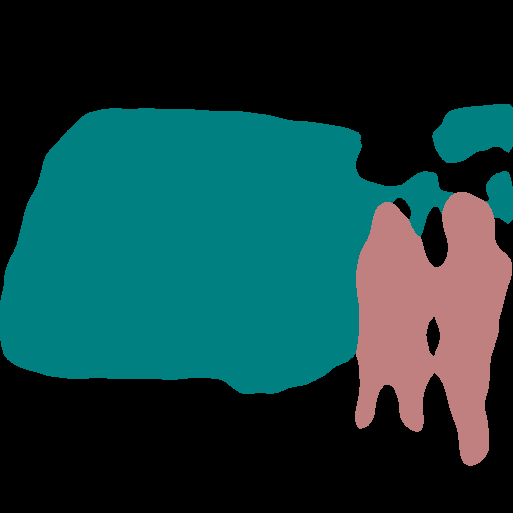}} &
            \centered{\includegraphics[width=.15\columnwidth]{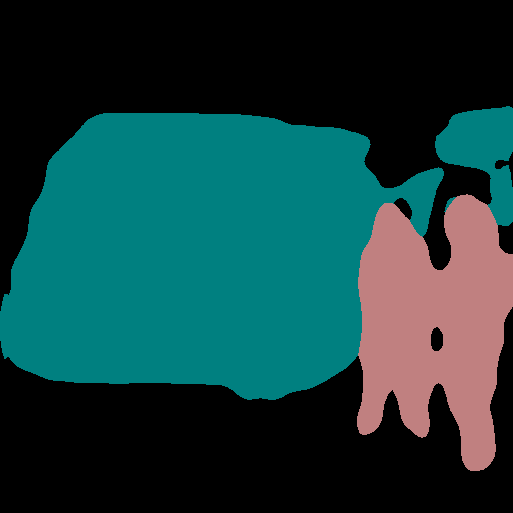}} &
            \centered{\includegraphics[width=.15\columnwidth]{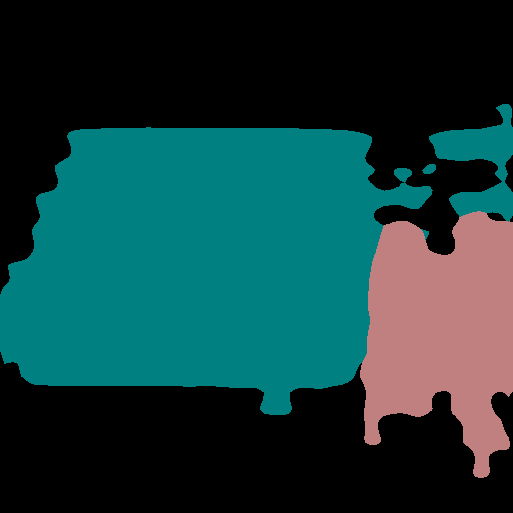}} \\ \hline
            \multicolumn{5}{c}{DeepGlobe} \\ \hline
            \includegraphics[width=.15\columnwidth]{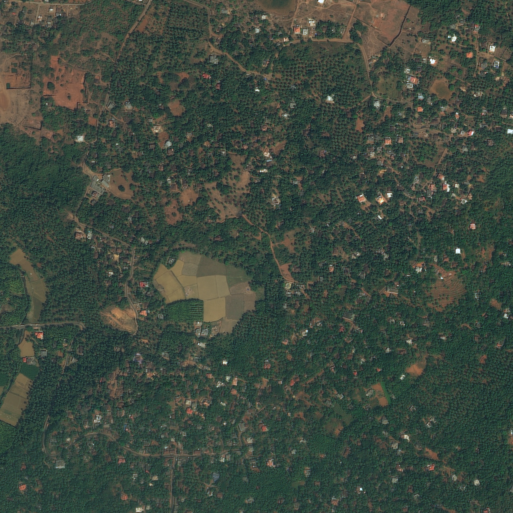} &
            \includegraphics[width=.15\columnwidth]{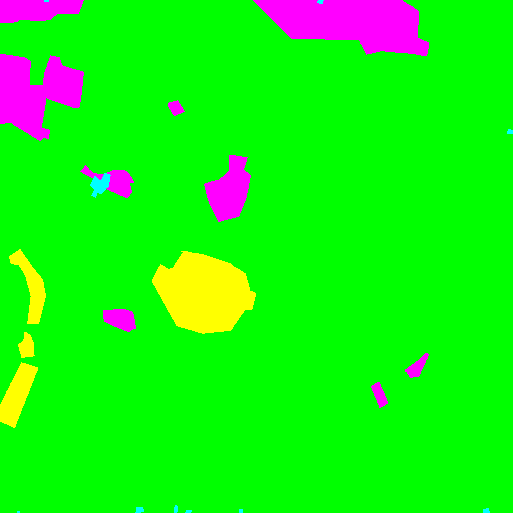} &
            \includegraphics[width=.15\columnwidth]{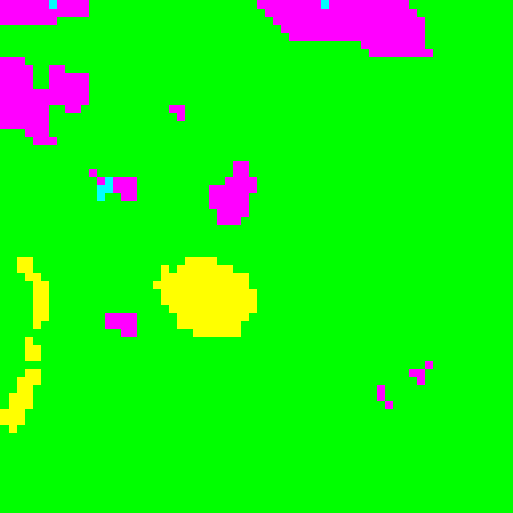} &
            \includegraphics[width=.15\columnwidth]{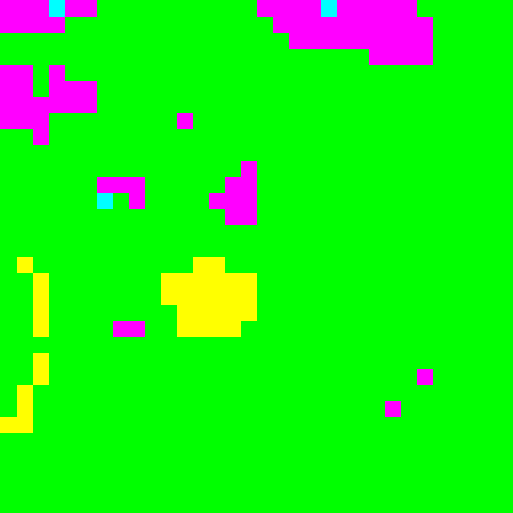} &
            \includegraphics[width=.15\columnwidth]{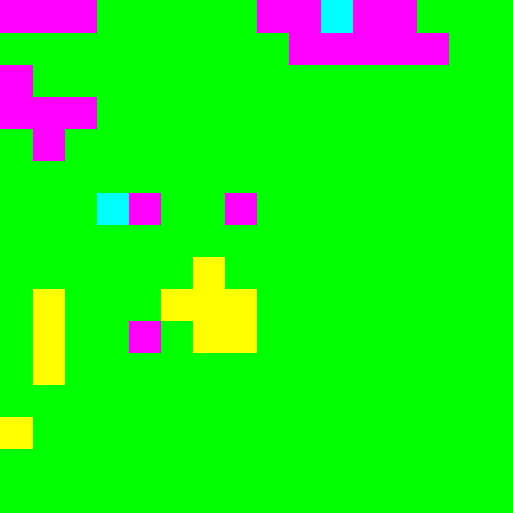} \\  
			\centered{DeepLab v3+} &
            \centered{\includegraphics[width=.15\columnwidth]{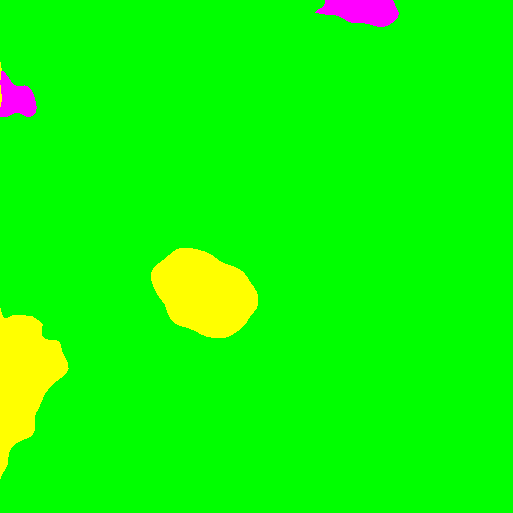}} &
            \centered{\includegraphics[width=.15\columnwidth]{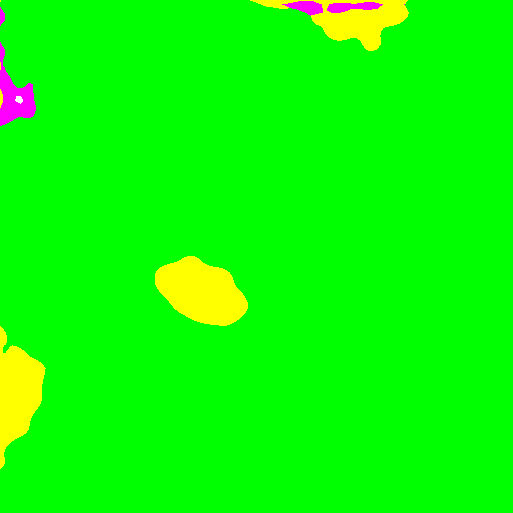}} &
            \centered{\includegraphics[width=.15\columnwidth]{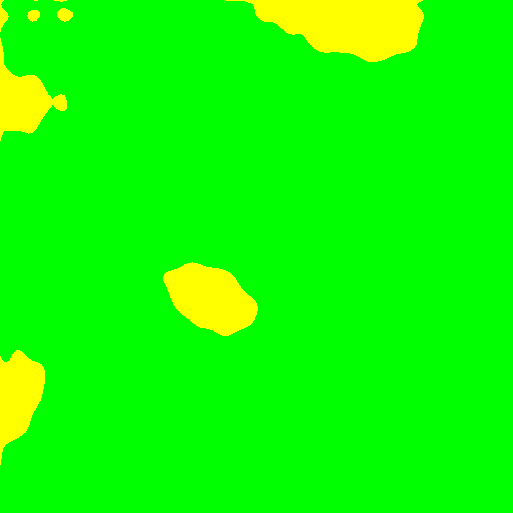}} &
            \centered{\includegraphics[width=.15\columnwidth]{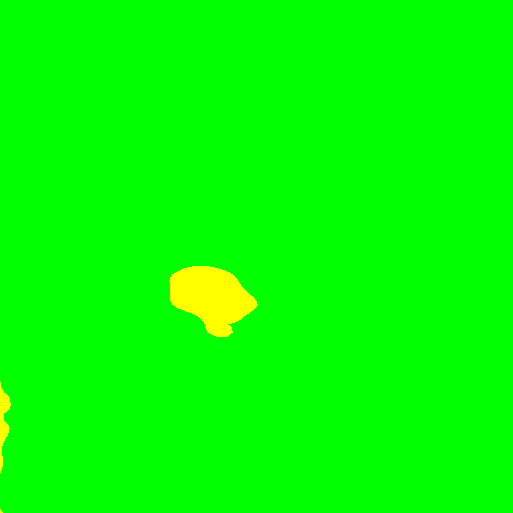}} \\
			\centered{Dan} &
            \centered{\includegraphics[width=.15\columnwidth]{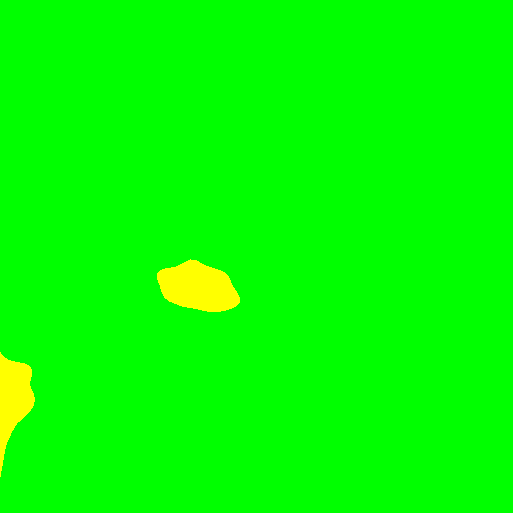}} &
            \centered{\includegraphics[width=.15\columnwidth]{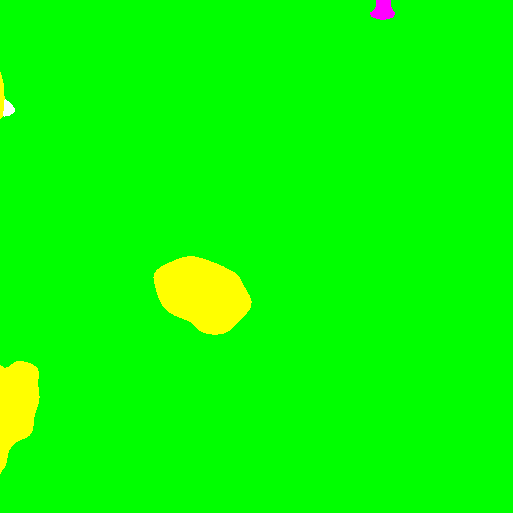}} &
            \centered{\includegraphics[width=.15\columnwidth]{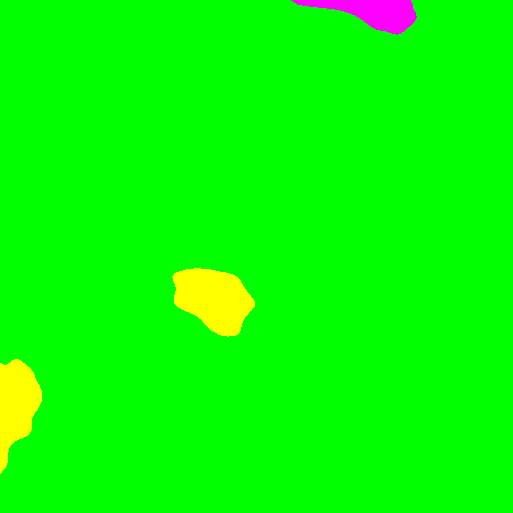}} &
            \centered{\includegraphics[width=.15\columnwidth]{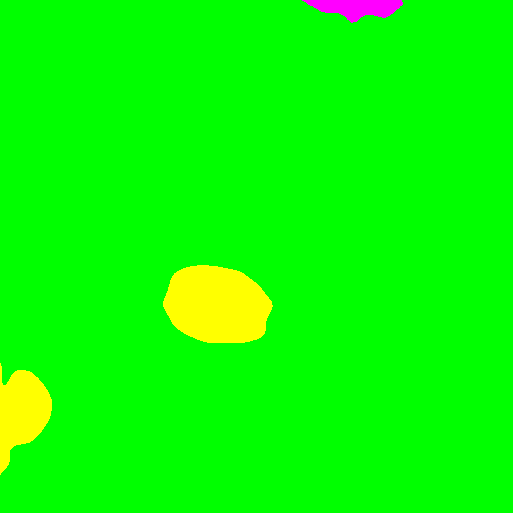}} \\ \hline
            \multicolumn{5}{c}{Cityscapes} \\ \hline
            \includegraphics[width=.15\columnwidth]{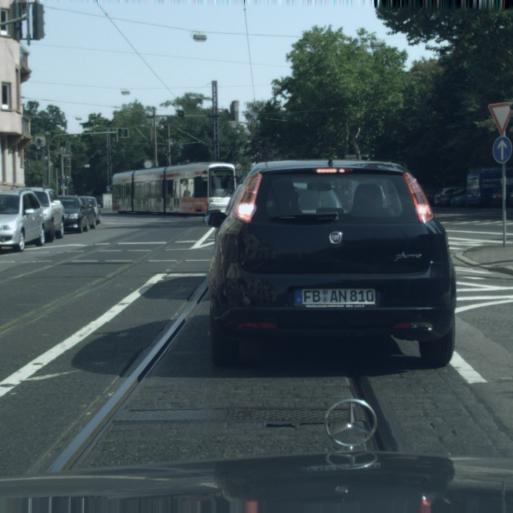} &
            \includegraphics[width=.15\columnwidth]{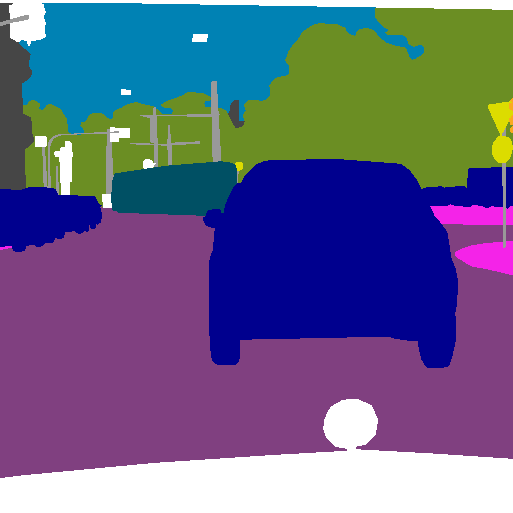} &
            \includegraphics[width=.15\columnwidth]{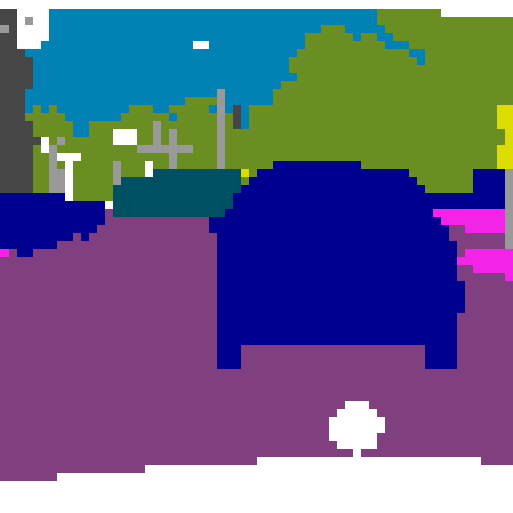} &
            \includegraphics[width=.15\columnwidth]{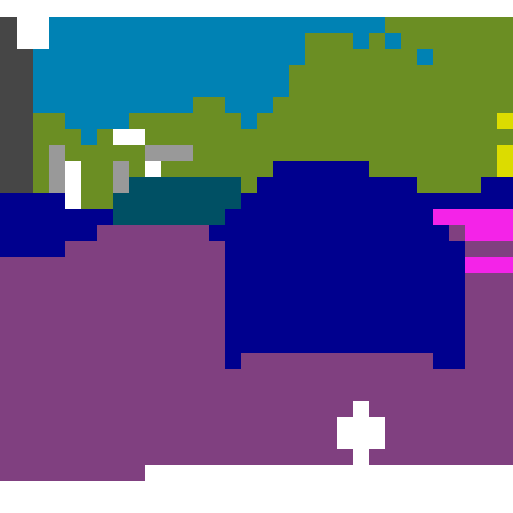} &
            \includegraphics[width=.15\columnwidth]{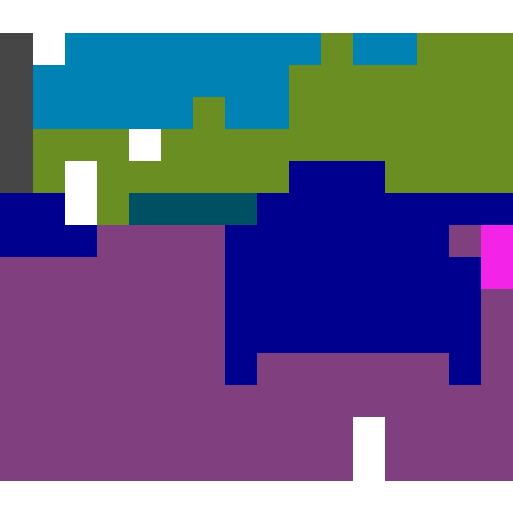} \\  
			\centered{DeepLab v3+} &
            \centered{\includegraphics[width=.15\columnwidth]{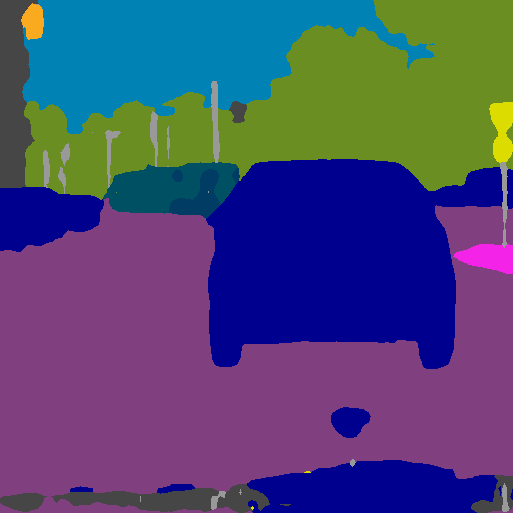}} &
            \centered{\includegraphics[width=.15\columnwidth]{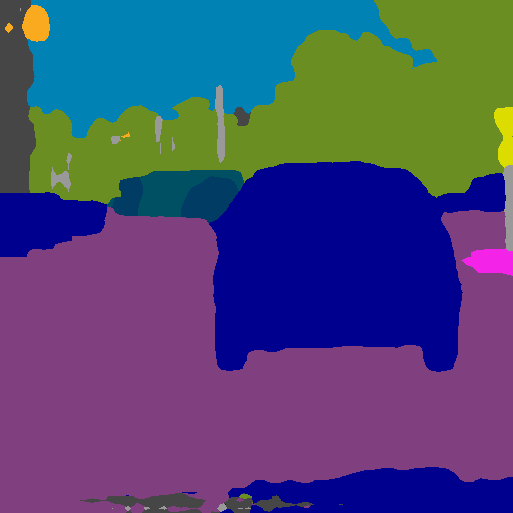}} &
            \centered{\includegraphics[width=.15\columnwidth]{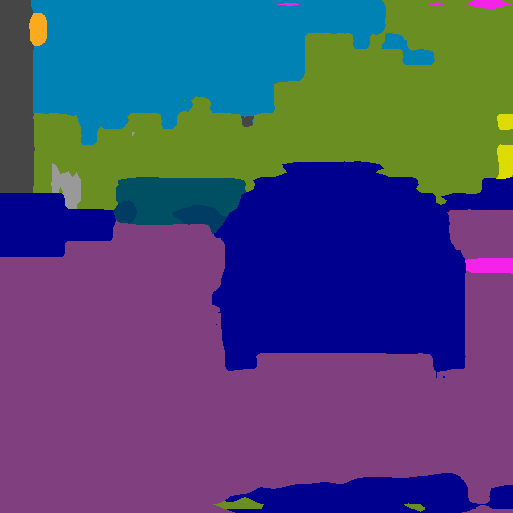}} &
            \centered{\includegraphics[width=.15\columnwidth]{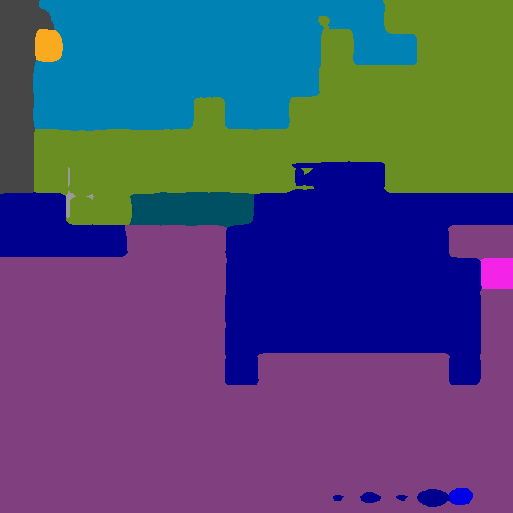}} \\
			\centered{Dan} &
            \centered{\includegraphics[width=.15\columnwidth]{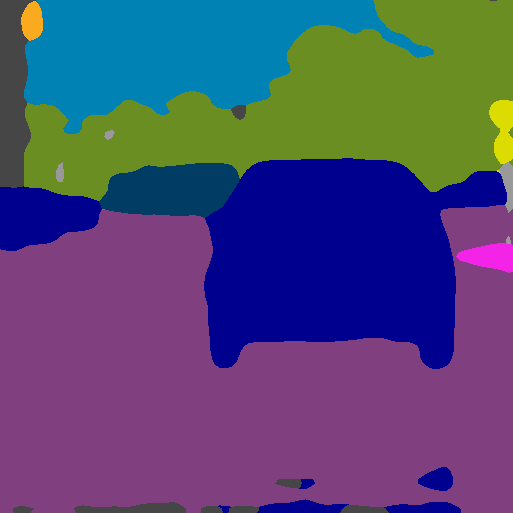}} &
            \centered{\includegraphics[width=.15\columnwidth]{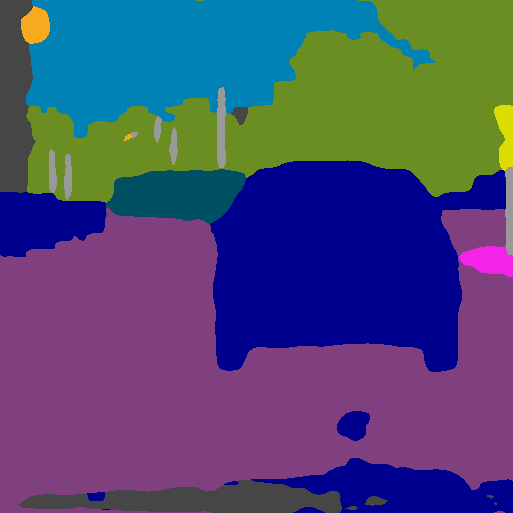}} &
            \centered{\includegraphics[width=.15\columnwidth]{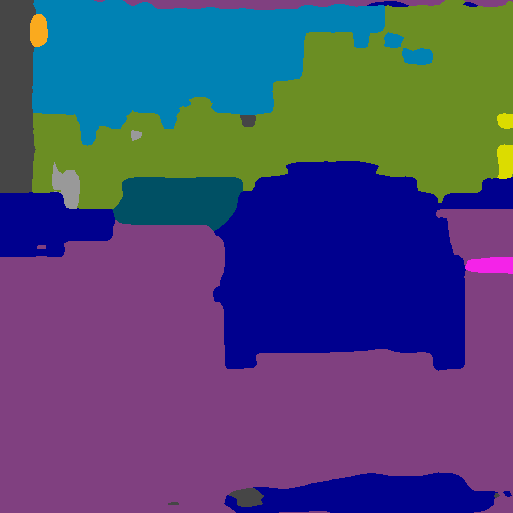}} &
            \centered{\includegraphics[width=.15\columnwidth]{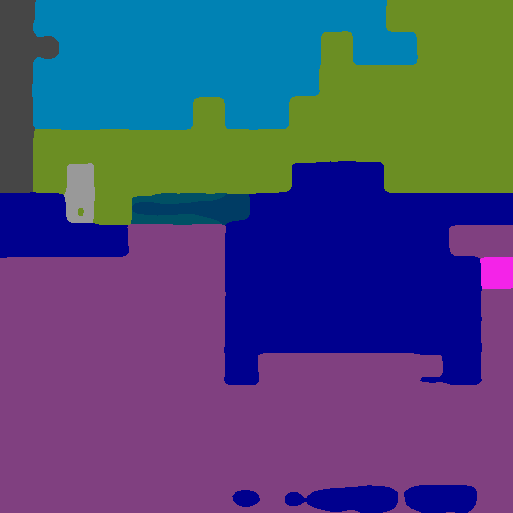}}
        \end{tabular}
        }
    \end{subfigure}
    \caption{Examples of the block-wise annotation and the prediction of DeepLab v3+ and Dan. For each dataset, we present the block-wise groundtruth annotations and the predictions of these models. The corresponding band limits ${\nu}_{max}$ are denoted at the top of annotation.}
    \label{Fig:block_annotation}
\end{figure}
\begin{table}[h]
\centering
\begin{subtable}[]{0.7\linewidth} 
    \resizebox{1.0\columnwidth}{!}{
    \begin{tabular}{ccccc}
    \multicolumn{1}{c|}{${\nu}_{max}$} & \multicolumn{1}{c|}{256} & \multicolumn{1}{c|}{32} & \multicolumn{1}{c|}{16} & 8 \\ \hline
    \multicolumn{5}{c}{DeepLab v3+} \\ \hline
    \multicolumn{1}{c|}{PASCAL} & \multicolumn{1}{c|}{78.5\%} & \multicolumn{1}{c|}{77.1\%} & \multicolumn{1}{c|}{74.2\%} & 67.4\% \\
    \multicolumn{1}{c|}{DeepGlobe} & \multicolumn{1}{c|}{55.0\%} & \multicolumn{1}{c|}{54.8\%} & \multicolumn{1}{c|}{53.8\%} & 52.5\% \\
    \multicolumn{1}{c|}{CityScapes} & \multicolumn{1}{c|}{67.8\%} & \multicolumn{1}{c|}{58.9\%} & \multicolumn{1}{c|}{50.5\%} & 38.4\% \\ \hline
    \multicolumn{5}{c}{DAN} \\ \hline
    \multicolumn{1}{c|}{PASCAL} & \multicolumn{1}{c|}{77.6\%} & \multicolumn{1}{c|}{76.5\%} & \multicolumn{1}{c|}{74.0\%} & 67.6\% \\
    \multicolumn{1}{c|}{DeepGlobe} & \multicolumn{1}{c|}{53.6\%} & \multicolumn{1}{c|}{53.4\%} & \multicolumn{1}{c|}{52.2\%} & 50.1\% \\
    \multicolumn{1}{c|}{CityScapes} & \multicolumn{1}{c|}{66.4\%} & \multicolumn{1}{c|}{58.3\%} & \multicolumn{1}{c|}{50.4\%} & 38.4\%
    \end{tabular}    
    }
    \caption{mIoU score.}
    \label{table:blk_annot:iou}
\end{subtable}

\begin{subtable}[]{0.7\linewidth} 
    \resizebox{1.0\columnwidth}{!}{
    \begin{tabular}{ccccc}
    \multicolumn{1}{c|}{v} & \multicolumn{1}{c|}{256} & \multicolumn{1}{c|}{32} & \multicolumn{1}{c|}{16} & 8 \\ \hline
    \multicolumn{5}{c}{DeepLab v3+} \\ \hline
    \multicolumn{1}{c|}{PASCAL} & \multicolumn{1}{c|}{0.0\%} & \multicolumn{1}{c|}{1.8\%} & \multicolumn{1}{c|}{5.4\%} & 14.1\% \\
    \multicolumn{1}{c|}{DeepGlobe} & \multicolumn{1}{c|}{0.0\%} & \multicolumn{1}{c|}{0.3\%} & \multicolumn{1}{c|}{2.3\%} & 4.6\% \\
    \multicolumn{1}{c|}{CityScapes} & \multicolumn{1}{c|}{0.0\%} & \multicolumn{1}{c|}{13.1\%} & \multicolumn{1}{c|}{25.4\%} & 43.4\% \\ \hline
    \multicolumn{5}{c}{DAN} \\ \hline
    \multicolumn{1}{c|}{PASCAL} & \multicolumn{1}{c|}{0.0\%} & \multicolumn{1}{c|}{1.4\%} & \multicolumn{1}{c|}{4.6\%} & 12.9\% \\
    \multicolumn{1}{c|}{DeepGlobe} & \multicolumn{1}{c|}{0.0\%} & \multicolumn{1}{c|}{0.3\%} & \multicolumn{1}{c|}{2.6\%} & 6.5\% \\
    \multicolumn{1}{c|}{CityScapes} & \multicolumn{1}{c|}{0.0\%} & \multicolumn{1}{c|}{12.2\%} & \multicolumn{1}{c|}{24.2\%} & 42.1\%
    \end{tabular}    
    }
    \caption{relative mIoU-drop.}
    \label{table:blk_annot:iou_drop}
\end{subtable}
\caption{Experimental results of using our proposed block-wise annotation for learning semantic segmentation.
}
\label{table:blk_annot}
\end{table}
As the band limit ${\nu}_{max}$ goes lower, mIoU score goes smaller and a positive mIoU drop is observed. Here we also observe a significant larger amount of mIoU drop on the Cityscapes dataset comparing to those on the PASCAL and DeepGlobe datasets, which consists to the trend of $R_{ce}({\nu}_{max})$ in Table~\ref{table:spectral_stat} that $R_{ce}({\nu}_{max})$ on the Cityscapes dataset are significantly larger than the other datasets.
Fig.~\ref{Fig:block_annotation_iou_R} illustrates the correlation between relative mIoU drops and $R_{ce}({\nu}_{max})$ for all experiments. The positive correlation between mIoU drop and $R_{ce}({\nu}_{max})$ agrees the correlation between CE and IoU score discussed in appendix~\ref{ssec:method_spectral_ioU}. Our studies show that the performance of the semantic segmentation network trained with the block-wise annotation strongly correlates to $R_{ce}({\nu}_{max})$. As a result, one can estimate the performance of the semantic segmentation network trained with the block-wise annotation by simply evaluating $R_{ce}({\nu}_{max})$ without thoroughly performing the experiments over all band limits.
\begin{figure}[h]
    \centering
    \includegraphics[width=0.8\columnwidth]{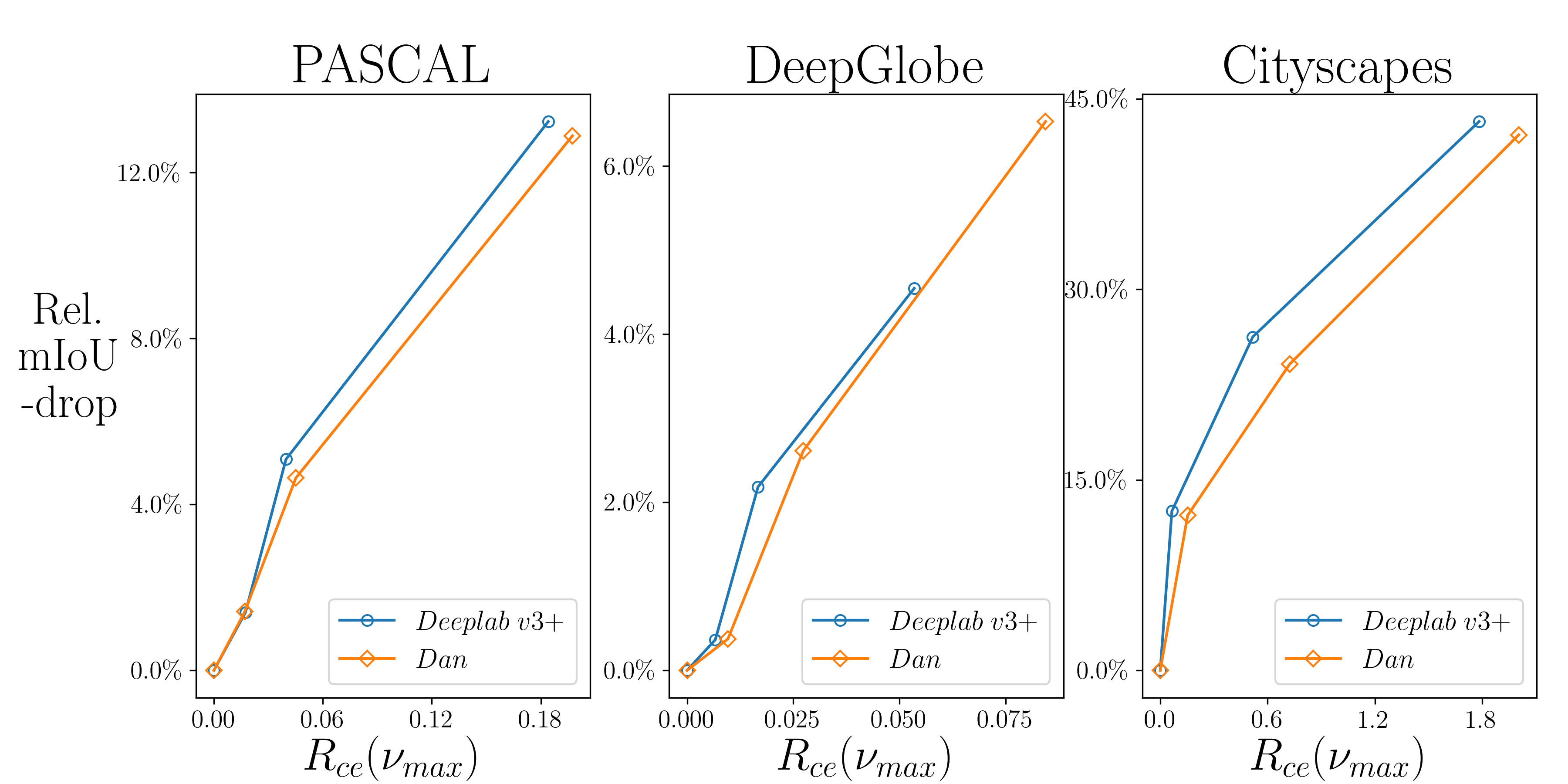} 
    \caption{The correlation between relative IoU-drop and the $R_{ce}({\nu}_{max})$ for block-wise annotation. 
    }
    \label{Fig:block_annotation_iou_R}
\end{figure}

In summary, the proposed spectral analysis enables the advanced analysis of weak annotation in the frequency domain. Our studies reveal the correlation between the segmentation performance and the LRG of segmentation maps. Based on our analysis and experiments, the block-wise annotation can be considered as a weak annotation when the block size is chosen according to the LRG size in the segmentation maps. Notably, these LRGs actually correspond to the coarse contour of instances in the segmentation maps, which are greatly utilized in the existing weak annotation. We provide the theoretical justification of the weak annotations by using our spectral analysis. Further research should be undertaken to investigate the spectral analysis upon the existing weak annotations~\cite{papandreou2015weakly,khoreva2017simple} in the future.

\section{Conclusion}\label{sec:conclusion}
Our proposed spectral analysis for semantic segmentation network correlate CE, IoU score and gradient back-propagation in the spectrum point of view. 
We first explicitly decompose CE and demonstrate the CE is mainly contributed by the low-frequency component of the segmentation maps, which associates with the features in CNNs at the same frequency. Furthermore, we proposed $R(\nu_{max})$ to estimate the efficacy of the LRG for segmentation maps. 
We test our theory on two applications: feature truncation and block annotation. Our results show that combination of the feature truncation and the network pruning can save computational cost significantly with small accuracy lost. In addition, the block annotation can potentially save more in labeling cost, since the network trained using the block-wise annotation in an efficient LRG performs close to the original network. The results from our experiments agree with our theoretical predictions based on $R(\nu_{max})$.
Lastly, despite the theoretical analysis and validation in this work, it remains unclear that how to determine the efficient band limit $\nu_{max}$ of the LRG for various datasets. It would be our future interests to estimate $\nu_{max}$ from the spectrum of groundtruth annotation.

\clearpage
\appendix
\section{Appendix}\label{sec:appendix}
\subsection{Related Work}\label{ssec:related work}
\subsubsection{Semantic Segmentation Neural Network}\label{ssec:relate:ssnn}
Among semantic segmentation neural networks (SSNN), Long~\etal first propose Fully Convolutional Neural Network (FCN)~\cite{long2015fully} that predicts the dense segmentation map by utilizing the skip-architecture, where the features of different granularities in the encoder are up-sampled and integrated in the decoder, yet still faces the challenge of acquiring accurate object boundaries in the segmentation map. 
The similar idea can be also observed in the U-Net \cite{ronneberger2015u}, which further adds dense skip-connections between the corresponding down-sampling and up-sampling modules of the same feature dimensions, in results the boundary localization is improved but not fully resolved yet. 
Other than skip connections, Chen~\etal propose the DeepLab models~ \cite{chen2014semantic,chen2017rethinking,chen2018encoder} that integrate the atrous spatial pyramid pooling module (ASPP), which utilizes the dilated convolutional layer composed of the filters at multiple sampling rates thus having the contextual information at the various spatial resolution, to boost the edge-response at object boundaries. Kou~\etal further propose Deep Aggregation Net (DAN) that utilize an aggregation decoder and progressively combines encoder features for final prediction to resolve the land cover segmentation across image scales. 
Besides of these SSNNs, extra modules such as dense conditional random field (dense CRF)~\cite{chen2014semantic,krahenbuhl2011efficient} and PointRend~\cite{kirillov2019pointrend} can be further applied to boost the edge-response near object boundaries while induce extra computational cost.
It is clear that improving edge-response near object boundaries becomes a main challenge of semantic segmentation while the cost of SSNNs grows dues to the dense decoder feature and post processing modules. This work investigate the spectral analysis and computation cost of DeepLab v3+ and DAN. We briefly review the cost of these SSNNs in the next section.

\subsubsection{Network Pruning}\label{ssec:relate:pruning}
Network pruning~\cite{liu2018rethinking,he2019asymptotic,Molchanov_2019_CVPR,blalock2020state,zhao2019variational,luo2017thinet,karnin1990simple,han2015learning} is proposed to reduce the cost of inference by removing the redundant network parameters. 
The redundant parameters are determined when either their contribution to output~\cite{Molchanov_2019_CVPR,luo2017thinet,zhao2019variational} or their norm~\cite{he2019asymptotic} are negligible. 
Noting that most of these pruning method are by hard pruning, \ie remove some weight values of filters~\cite{han2015learning} or completely remove the whole filters ~\cite{luo2017thinet}, while potentially degrading the capacity of networks. In contrast, He~\etal~\cite{he2019asymptotic} propose the soft pruning method that dynamically set redundant parameters to zero while keep the network capacity. This enables the compressed network to have a larger optimization space and make it easier for the model to learn from the training data, and achieve higher accuracy. 

Despite the success of these methods for accelerating network, we would like to point out that the existing pruning methods are solely investigated upon image classification~\cite{krizhevsky2009learning,russakovsky2015imagenet} instead of other task, such as semantic segmentation or image generation. We further investigate the application of pruning methods on semantic segmentation in this work. Noting the the segmentation networks typically have huge parameters in encoder while negligible parameters in decoder~\cite{chen2018encoder,kuo2018dan,long2015fully}. However, the computational costs of decoders are often comparable to those of encoders since it up-samples the features for the dense segmentation map and results in large features for computation. For example, the encoder of DeepLab v3+~\cite{chen2018encoder} has 95.6 billion FLOPs (floating-point operations) and 60.1 million parameters. In contrast, its decoder has 43.4 billion FLOPs while only has 1.3 million parameters. Similarly, the encoder of DAN~\cite{kuo2018dan} has 95.6 billion FLOPs with 60.1 million parameters while the decoder has 11.1 billion FLOPs with only 0.4 million parameters. The proposed feature truncation in section~\ref{ssec:exp_feat_trunc} is thus expected to effectively reduce the computational cost. Moreover, one can combine feature truncation with the typical network pruning method to reduce the computation cost in two different aspects, \ie the feature size and the redundant parameters.

\subsubsection{Spectral Analysis}\label{ssec:relate:spectral}
The existing works of spectral analysis demonstrate that the network tends to learn the low-frequency component of target signal in the regression of the uniformly distributed data with various frequencies~\cite{rahaman2018spectral,ronen2019convergence,luo2019theory,yang2019fine,xu2019frequency}. Such tendency is known as spectral bias~\cite{rahaman2018spectral} or Frequency Principle\cite{xu2019frequency}. 
More specifically, these works found that the network tend to learn low-frequency signal in the earlier training stage. Ronen~\etal~\cite{ronen2019convergence} further provide the theoretical explanation based normalized training data that is uniformly distributed on a hypersphere. Under same assumption of data distribution, Yang and Salman~\cite{yang2019fine} further investigate the eigen function of neural tangent kernel (NTK)~\cite{jacot2018neural} and demonstrate that the eigenvalue of NTK decrease as the frequency increases. This provide further theoretical insight and justify the spectral bias that the learning of networks converge faster for low-frequency signals. 

So far, the existing works investigate the spectral bias solely under the normalized training data with uniform distribution over frequency regime. This work further extends spectral analysis to semantic segmentation, where the target data is non-uniformly distributed. Furthermore, these works mostly study the convergence speed for each frequency regime while this work focus on the learned distribution of networks at final training stage. This helps us to estimate the capacity of models, in the sense of frequency, under the spectral bias in semantic segmentation.

\subsection{Fourier transform of spatial integral}
\begin{lemma}\label{lemma:ft_prod}
Given two functional $Y(t)$ and $B(t)$ in spatial domain $t$, the overlapping integral ${\int} Y(t){B}(t)dt$ can be transformed into the frequency domain $\nu$ as 
\begin{equation}
    \begin{aligned}
        {\int} Y(t){B}(t)\,dt = {\int} y(z)\,{b}( - z)\,dz
    \end{aligned}
\end{equation}
where $y(\nu)=\mathcal{F}(Y(t))$ and $b=\mathcal{F}(B)$.
\end{lemma}

\begin{proof}
By the convolution lemma, integral ${\int} Y(t){B}(t)dt$ can be written as
\begin{equation}
    \begin{aligned}
        {\int} Y(t){B}(t)\,dt & ={\int} {{\cal F}^{-1}}(y(\nu) \otimes b(\nu))\,dt
    \end{aligned}
    \label{Eq:spectral_pxprod0}
\end{equation}
; where $ \otimes $ denotes the convolution operation as $y(\nu) \otimes b(\nu) = {\int} y(z)\,b(\nu  - z)\,dz$; 
${{\mathcal F}^{-1}}(y(\nu)) = {\int} y(\nu) e^{2j\pi\nu t}d\nu$ is the inverse Fourier transform operator; $j = \sqrt{-1}$. Eq.~\ref{Eq:spectral_pxprod0} can now be written as
\begin{equation}
    \begin{aligned}
        {\int} Y(t){B}(t)\,dt & = {\int}{\int} y(\nu) \otimes b(\nu)e^{2j\pi\nu t}\,d\nu \,dt
        \\& = {\int} {\int} {\int} e^{2j\pi \nu t}\,y(z)\,{b}(\nu  - z)\,dt\,d\nu \,dz
    \end{aligned}
    \label{Eq:spectral_pxprod1}
\end{equation}
By the orthogonality of Fourier basis, we have ${\int} e^{2j\pi \nu t}dt = D_{0}(\nu)$, where $D_{0}(\nu)$ is the Dirac delta function:
\begin{equation}
    \begin{aligned}
        D_{0}(\nu) = 
\begin{cases}
    \infty,& \text{if } \nu=0 \\
    0,              & \text{otherwise}
\end{cases}        
    \end{aligned}
    \label{Eq:delta}
\end{equation}
and its integral property is $\int y(\nu)D_{0}(\nu)d\nu=y(0)$. 
Hence, Eq.~\ref{Eq:spectral_pxprod1} is given as 
\begin{equation}
    \begin{aligned}
        {\int} Y(t){B}(t)\,dt & = {\int} {\int} y(z)\,{b}(\nu  - z)\,D_{0}(\nu)\,d\nu \,dz \\
        & = {\int} y(z)\,{b}( - z)\,dz
    \end{aligned}
\end{equation}
\end{proof}

\begin{lemma}\label{lemma:ft_sum}
Given functional $Y$ in spatial domain $t$, the integral ${\int} Y(t)dt$ can be transformed into the frequency domain $\nu$ as 
\begin{equation}
    \begin{aligned}
        {\int} Y(t)dt & = y(0);
    \end{aligned}
    \label{Eq:spectral_pxsum}
\end{equation}
where $y(\nu)=\mathcal{F}(Y(t))$.
\end{lemma}

\begin{proof}
The proof follows similar process as for lemma~\ref{lemma:ft_prod}, as follows.
\begin{equation}
    \begin{aligned}
        {\int} Y(t)dt & ={\int} {{\cal F}^{-1}}(y(\nu))dt 
        \\& = {\int}{\int} y(\nu)e^{2j\pi \nu t}d\nu dt
        \\& = {\int} y(\nu){D_{0}(\nu)}d\nu
        \\& = y(0)
    \end{aligned}
\end{equation}
\end{proof}

\subsection{Proof of Theorem~\ref{theorem:ce}}\label{ssec:proof_spectral_ce}
\ce*
\begin{proof}
Given $Y$ and $B$, $\mathcal{L}_{CE}$ the cross-entropy is
\begin{equation}
    \begin{aligned}
         \mathcal{L}_{CE} & = -  { \sum _{c}} {\int} B(t,c) {log} (   { \frac {e^{Y(t,c)}}{{ \sum _{c}} e^{Y(t,c)}}}) dt 
        \\ & = - { \sum _{c}} {\int} B(t,c)(Y(t,c) - {Y_{p}}(t))dt,
    \end{aligned}
\end{equation}
where ${Y_{p}}(t)={log}({ \sum _{c}} e^{Y(t,c)})$. 
For all $Y(t)\in  Y(t,c)$ and ${B}(t)\in B(t,c)$, the integral ${\int} Y(t){B}(t)dt$ can be transformed to the frequency domain $\nu$ as follows. (See lemma~\ref{lemma:ft_prod} of appendix)
\begin{equation}
    \begin{aligned}
        {\int} Y(t){B}(t)\,dt = {\int} y(z)\,{b}( - z)\,dz
    \end{aligned}
\end{equation}
where $y(\nu)$ and $b(\nu)$ are the spectrum of the segmentation logits and that of the groundtruth annotations, respectively. The $\mathcal{L}_{CE}$ in Eq.~\ref{Eq:ce} is hence given by
\begin{equation}
    \begin{aligned}
         \mathcal{L}_{CE} & = - { \sum _{c}} {\int} B(t,c)(Y(t,c) - {Y_{p}}(t))dt
              \\ & = { \sum _{c}} {\int} {b}(-\nu, c)({y_{p}}(\nu) - y(\nu, c))d\nu.
    \end{aligned}
    \label{Eq:spectral_ce_final}
\end{equation}
The discrete integral of  Eq.~\ref{Eq:spectral_ce_final} gives us the  decomposition of the $\mathcal{L}_{CE}$  over frequency domain $\nu$ as following
\begin{equation}
    \begin{aligned}
        \mathcal{L}_{CE} & = { \sum _{\nu}} { \sum_{c} }{b}(-\nu, c)({y_{p}}(\nu) - y(\nu, c))
          \\ & = { \sum _{\nu}} \mathcal{L}_{ce}(\nu),
    \end{aligned}
    \label{Eq:spec_ce_final_discrete}
\end{equation}
where 
\begin{equation}
    \begin{aligned}
        \mathcal{L}_{ce}(\nu) = { \sum_{c} }{b}(-\nu, c)({y_{p}}(\nu) - y(\nu, c)).
    \end{aligned}
\end{equation}
\end{proof}

\subsection{Spectral Analysis of Intersection-over-Union Score}\label{ssec:method_spectral_ioU}
Given the segmentation logits $Y$ and the groundtruth annotation $B$, the intersection-over-union (IoU) score is typically defined as $\frac{|B \cap S|}{ |B \cup S|}$, where $S$ is the segmentation output $S(t,c)={ \frac {e^{Y(t,c)}}{{ \sum _{c}} e^{Y(t,c)}}}$. 
It is common to train the network with CE and evaluate the network performance based on IoU scores. This section aims to analyze the formalism of IoU score in frequency domain and shed some light to the reason why IoU scores can be increased when the CE is decreased.

In order to analyze the IoU score in the frequency domain, we extend the above definition to the continuous space as follows: 
\begin{equation}
    \begin{aligned}
    IoU( {S,B} ) & = \frac{{\int B(t)S(t)} \,dt}{{\int B(t) + S(t)\,dt - \int B(t)S(t)\,dt}}
    \\ & = \frac{1}{{\frac{{\int B(t) + S(t)\,dt}}{{\int B(t)S(t)\,dt}} - 1}}
    \end{aligned}
    \label{Eq:iou}
\end{equation}
where $t$ denotes pixel indexes. Eq.~\ref{Eq:iou} holds for each object class $c$. Here we skip $c$ for simplicity. Notably, this definition is equivalent to the origin definition of IoU score for the binarized segmentation maps. The components in Eq.~\ref{Eq:iou} can be written as follows (see lemma~\ref{lemma:ft_prod} and lemma~\ref{lemma:ft_sum} of appendix), 
\begin{equation}
    \begin{aligned}
        \int B(t)S(t)\,dt & = \int s(\nu)b(- \nu)d\nu, \text{~and}~ \\
        {\int} B(t)+S(t)dt &= b(0)+s(0),
    \end{aligned}
\end{equation}
where $s(\nu)={\cal F}(S(t))$. As a result, IoU score can be written as
\begin{equation}
    \begin{aligned}
        & IoU({s,b}) =  \frac{1}{{\frac{{ {s(0) + b(0)} }}{{{\int} s(\nu)b(- \nu)d\nu}} - 1}}
    \end{aligned}
    \label{Eq:spectral_iou}
\end{equation}
and it is composed of two terms: ${s(0) + b(0)}$ and ${\int} s(\nu)b(- \nu)d\nu = \int B(t)S(t)\,dt$. 
It can be seen that the IoU score can not be explicitly decomposed as the case for CE in Eq.~\ref{Eq:spec_ce_final_discrete} due to the non-linearity of Eq~\ref{Eq:spectral_iou}. On the other hand, noting that the latter term, \ie ${\int} s(\nu)b(- \nu)d\nu = \int B(t)S(t)\,dt$, positively correlates to ${\int} B(t){log}(S(t))dt$ the component of CE in Eq.~\ref{Eq:ce} since the $log$ function is monotonically increasing. 
In addition, ${\int} B(t){log}(S(t))dt$ the component of CE can be approximated as 
\begin{equation}
    \begin{aligned}
        {\int} B(t){log}(S(t))dt &= {\int} B(t)(S(t)-1)dt
        \\ &= {\int} s(\nu)b(- \nu)d\nu - b(0)
    \end{aligned}
    \label{Eq:spectral_iou_approx}
\end{equation}
by the Taylor expansion of the $log$ function. In this case, the component of CE only deviates from ${\int} s(\nu)b(- \nu)d\nu$ by $b(0)$ that is independent of $s(\nu)$. Hence, minimal $\mathcal{L}_{CE}$ maximizes $\int B(t)S(t)\,dt$ as well as IoU score. In addition, it immediately follows that IoU is mainly contributed by low-frequency components as $\mathcal{L}_{CE}$. 

\subsection{Spectral Analysis of Relative IoU-drop}\label{ssec:method_trunc_ioU}
Following the discussion in section~\ref{ssec:method_spectral_ioU}, we now discuss the IoU loss caused by using the LRG for prediction. Furthermore, we aim to investigate the correlation between IoU loss and the $R_{ce}({\nu}_{max})$ in Eq.~\ref{Eq:r_max}.
Following the spectral analysis in Eq.~\ref{Eq:spectral_iou}, let truncated IoU be 
\begin{equation}
    \begin{aligned}
        & \widehat{IoU}({\nu}_{max}) =  \frac{1}{{\frac{{ {s(0) + b(0)} }}{{{\int}^{{\nu}_{max}}_0 s(\nu)b(- \nu)d\nu}} - 1}}
    \end{aligned}
    \label{Eq:spectral_iou_bandlimit}
\end{equation}
where ${\nu}_{max}$ is the band limit defined in Eq.~\ref{Eq:r_max} as well as in Eq.~\ref{Eq:spectral_ce_bandlimit}.
Plugging in Eq.~\ref{Eq:spectral_iou} and Eq.~\ref{Eq:spectral_iou_bandlimit}, the IoU loss caused by LRG is thus given as
\begin{equation}
    \begin{aligned}
        \Delta{IoU} &= IoU - \widehat{IoU}({\nu}_{max}) 
        \\ &=  \frac{1}{{\frac{{ {s(0) + b(0)} }}{{{\int} s(\nu)b(- \nu)d\nu}} - 1}} - \frac{1}{{\frac{{ {s(0) + b(0)} }}{{{\int}^{{\nu}_{max}}_0 s(\nu)b(- \nu)d\nu}} - 1}},
    \end{aligned}
\end{equation}
where $IoU$ is $IoU(s,b)$ in Eq.~\ref{Eq:spectral_iou}. 
For simplicity, let $c = s(0) + b(0)$, $p_1 = {{{\int} s(\nu)b(- \nu)d\nu}}$, and $p_2 = {{{\int}^{{\nu}_{max}}_0 s(\nu)b(- \nu)d\nu}}$. The equation can then be simplified as
\begin{equation}
    \begin{aligned}
        \Delta{IoU} &= \frac{1}{\frac{c}{p_1} - 1} - \frac{1}{\frac{c}{p_2} - 1}
        \\ &= \frac{p_1}{c-p_1} - \frac{p_2}{c-p_2}
        \\ &= \frac{c}{(c-p_1)(c-p_2)} (p_1-p_2)
        \\ &= \frac{1}{\frac{c}{p_1} - 1} \frac{c}{c - p_2} \frac{p_1-p_2}{p_1}
        \\ &= IoU \frac{c}{c - p_2} \frac{p_1-p_2}{p_1}        
        \\ \frac{\Delta{IoU}}{IoU} &= \frac{c}{c - p_2} \frac{p_1-p_2}{p_1},
    \end{aligned}
    \label{Eq:spectral_iou_simplified}
\end{equation}
where $\frac{\Delta{IoU}}{IoU}$ is the relative IoU-drop. To simplify the discussion, let us consider only the one class performance. Namely, let
\begin{equation}
    \begin{aligned}
        \mathcal{L}_{CE} &= {\int} B(t){log}(S(t))dt 
        \\& = {\int} s(\nu)b(- \nu)d\nu - b(0)
        \\& = p_1 - b(0)
        \\ p_1 &= \mathcal{L}_{CE} + b(0)
    \end{aligned}
    \label{Eq:spectral_ce_approx_p1}
\end{equation}
Similarly, 
\begin{equation}
    \begin{aligned}
        \widehat{\mathcal{L}}_{ce}({\nu}_{max}) & = {\int}^{{\nu}_{max}}_0 s(\nu)b(- \nu)d\nu - b(0)
        \\& = p_2 - b(0)
        \\ p_2 &= \widehat{\mathcal{L}}_{ce}({\nu}_{max}) + b(0)
    \end{aligned}
    \label{Eq:spectral_ce_approx_p2}
\end{equation}
by the approximation in Eq.~\ref{Eq:spectral_iou_approx}. 
Without loss of generality, let us assume $b(0) = 0$. Plugging Eq.~\ref{Eq:spectral_ce_approx_p1} and Eq.~\ref{Eq:spectral_ce_approx_p2} into Eq.~\ref{Eq:spectral_iou_simplified},
\begin{equation}
    \begin{aligned}
        \frac{\Delta{IoU}}{IoU} &= |\frac{\Delta{IoU}}{IoU}|
        \\ &= |\frac{c}{c - p_2} \frac{p_1-p_2}{p_1}|
        \\ &= |\frac{1}{1 - \frac{\widehat{\mathcal{L}}_{ce}({\nu}_{max})}{s(0)}} \frac{ \mathcal{L}_{CE} - \widehat{\mathcal{L}}_{ce}({\nu}_{max})}{\mathcal{L}_{CE}}|
        \\ &= |\frac{1}{1 - \frac{\widehat{\mathcal{L}}_{ce}({\nu}_{max})}{s(0)}}| R_{ce}({\nu}_{max})
    \end{aligned}
    \label{Eq:spectral_ce_approx_sup}
\end{equation}
This demonstrates the explicit correlation between $\frac{\Delta{IoU}}{IoU}$ the relative IoU-drop and $R_{ce}({\nu}_{max})$, which has a positive slope if the approximation in Eq.~\ref{Eq:spectral_iou_approx} holds. 

\subsection{Spectral Analysis of Boundary Intersection-over-Union Score}\label{ssec:method_spectral_bioU}
Following the notation in section~\ref{ssec:method_spectral_ioU}, the boundary 
intersection-over-union (Boundary IoU)~\cite{cheng2021boundary} score is defined as 
\begin{equation}
    \begin{aligned}
        boundary\,IoU = \frac{|(B \cap B_d) \cap (S \cap S_d)|}{ |(B \cap B_d) \cup (S \cap S_d)|},
    \end{aligned}
\end{equation}
where $S_d$ and $B_d$ denote the pixels in the boundary region of $S$ and $B$, respectively; $d$ is the width of boundary region. 
Compared to IoU score, such evaluation metric is shown to be sensitive to the boundary especially for the large object. In addition to its sensitivity to object boundary, this section reveals its theoretical insight and demonstrate that it's mainly contributed by the low-frequency component of segmentation map. 

Without loss of generality, we analyze the 1 dimensional case of boundary IoU. 
We consider the binary segmentation map as follows,
\begin{equation}
    \begin{aligned}
        S &= H(t - t_{s0}) - H(t - t_{s1}), t_{s0} < t_{s1} \\
        B &= H(t - t_{b0}) - H(t - t_{b1}), t_{b0} < t_{b1}
    \end{aligned}
\end{equation}
where $H$ is the Heaviside function; $t_{s0}$ and $t_{s1}$ are the boundary pixels of $S$; $t_{b0}$ and $t_{s1}$ are the boundary pixels of $B$; 
we model the boundary region of segmentation map by two gaussian function for each boundary edge. Namely, 
\begin{equation}
    \begin{aligned}
        S_{\partial\Omega} &= S_d \cap S =e^{- \frac {(t - t_{s0})^{2}}{2\sigma^2}} + e^{ - \frac {(t -
        t_{s1})^{2}}{2\sigma^2}} \\
        B_{\partial\Omega} &= B_d \cap B = e^{- \frac {(t - t_{b0})^{2}}{2\sigma^2}} + e^{ - \frac {(t - t_{b1})^{2}}{2\sigma^2}}
    \end{aligned}
    \label{Eq:Omega}
\end{equation}
where $\sigma = \sigma(d)$ is the width of gaussian associating with the $d$ in $S_d$ and $B_d$ mentioned above. We have their Fourier transform as
\begin{equation}
    \begin{aligned}
        s_{\partial\Omega} &= (e^{- j\,t_{s0}\,\nu} + e^{- j\,t_{s1}\,\nu})\,e^{ - \frac {\nu ^{2}\,\sigma ^{2}}{2}} \\
        b_{\partial\Omega} &= (e^{- j\,t_{b0}\,\nu} + e^{- j\,t_{b1}\,\nu})\,e^{- \frac {\nu ^{2}\,\sigma ^{2}}{2}}
    \end{aligned}
    \label{Eq:omega}
\end{equation}

Following similar deduction as in Eq.~\ref{Eq:iou} and Eq.~\ref{Eq:spectral_iou}, we have
\begin{equation}
    \begin{aligned}
        boundary\,IoU({s,b}) &= \frac{1}{{\frac{{ {s_{\partial\Omega}(0) + b_{\partial\Omega}(0)} }}{{{\int} s_{\partial\Omega}(\nu)b_{\partial\Omega}(- \nu)d\nu}} - 1}},
    \end{aligned}
    \label{Eq:spectral_biou}
\end{equation}
where 
\begin{equation}
    \begin{aligned}
        &{{\int} s_{\partial\Omega}(\nu)b_{\partial\Omega}(- \nu)d\nu} = \frac {\sqrt{\pi }}{2\sigma} ( \\
        & e^{ - \frac {( - t_{s0} + t_{b0})^{2}}{4\sigma ^{2}}}\mathrm{erf}(\nu \sigma - j{\frac {
        (t_{b0} - t_{s0})}{2\sigma }} ) \\
        + & e^{- \frac {( - t_{s0} + t_{b1})^{2}}{4\sigma ^{2}}}\mathrm{erf}(\nu \sigma - j{\frac {(t_{b1} - t_{s0})}{2\sigma }} ) \\
        + & e^{- \frac {( - t_{s1} + t_{b0})^{2}}{4\sigma ^{2}}}\mathrm{erf}(\nu \sigma - j{ \frac {(t_{b0} - t_{s1})}{2\sigma }} ) \\
        + & e^{- \frac {( - t_{s1} + t_{b1})^{2}}{4\sigma ^{2}}}\mathrm{erf}(\nu \sigma - j{ \frac {(t_{b1} - t_{s1})}{2\sigma }} ))
    \end{aligned}
    \label{Eq:spectral_overla_integral}
\end{equation}
by plugging the Eq.~\ref{Eq:omega}; $\mathrm{erf}$ is the error function.
Similar to Eq.~\ref{Eq:spectral_iou}, Eq.~\ref{Eq:spectral_biou} consists of the zero-frequency part, \ie ${s_{\partial\Omega}(0) + b_{\partial\Omega}(0)}$, and the non-zero frequency part, \ie Eq.~\ref{Eq:spectral_overla_integral}. We focus on analyzing the non-zero frequency part to further reveal the sensitivity of boundary IoU with respect to these frequency regime. Eq.~\ref{Eq:spectral_overla_integral} can be further approximated as 
\begin{equation}
    \begin{aligned}
        &{{\int} s_{\partial\Omega}(\nu)b_{\partial\Omega}(- \nu)d\nu} = \frac {\sqrt{\pi }}{2\sigma}\mathrm{erf}(\nu \sigma)( 
        e^{ - \frac {( - t_{s0} + t_{b0})^{2}}{4\sigma ^{2}}}
        \\ & + e^{- \frac {( - t_{s0} + t_{b1})^{2}}{4\sigma ^{2}}} 
        + e^{- \frac {( - t_{s1} + t_{b0})^{2}}{4\sigma ^{2}}}
        + e^{- \frac {( - t_{s1} + t_{b1})^{2}}{4\sigma ^{2}}})
    \end{aligned}
    \label{Eq:spectral_overla_integral_approx}
\end{equation}
by using the expansion of erf function
\begin{equation}
    \begin{aligned}
        & \mathrm{erf}(\nu \sigma - jC) \simeq \mathrm{erf}(\nu \sigma) +
        \\ & \frac{e^{-\nu^2 \sigma^2}}{2\pi \nu \sigma}(1-cos(2\nu \sigma C)+j\,sin(2\nu \sigma C))
        \\ & \simeq \mathrm{erf}(\nu \sigma)
    \end{aligned}
    \label{Eq:erf_approx}
\end{equation}
It follows immediately from Eq.~\ref{Eq:spectral_overla_integral_approx} that ${{\int} s_{\partial\Omega}(\nu)b_{\partial\Omega}(- \nu)d\nu}$ is mainly contributed by low-frequency regime dues to erf function. 
This implies that boundary IoU is also mainly contributed by low-frequency regime while being sensitive to the object boundary.

\subsection{Gradient propagation for a convolution layer}\label{ssec:sgd_conv}
Consider a convolution layer consists of the convolutional kernel $K(t)$ and the soft-plus activation function $\sigma(z(t))=log(1+e^{z(t)})$; $t$ is the spatial location. Let $X$ denote the input, the output of convolution layer is written as 
\begin{equation}
\begin{aligned}
    Z(t) &= K(t) \otimes X(t) \\
    Y(t) &=\sigma(Z(t))
\end{aligned}
\end{equation}

\begin{lemma}\label{lemma:conv_grad}
Assuming $K(t)$ is small and $|X(t)| < 1$, the spectral gradient can be approximated as
\begin{equation}
    \begin{aligned}
         \frac{\partial y(\nu_i)}{\partial x(\nu_j)} \simeq \frac{1}{2}k( \nu_j){{\delta}_{\nu_j}(\nu_i)},
    \end{aligned}
\end{equation}
where $z(\nu), k(\nu)$, $x(\nu)$ and ${\delta}_{\nu_j}(\nu_i)$ are ${\cal F}(Z(t)), {\cal F}(K(t))$, ${\cal F}(X(t))$ and the Kronecker delta function, respectively.
\end{lemma}

\begin{proof}
The spectral gradient of a convolution layer consists of the spectral gradient for the convolution operator and that for the activation function. We will show two gradient and combine it in the end of derivation.

For the convolution operator, it can be written as in the frequency domain ${z}(\nu) = k(\nu)x(\nu)$, where $z(\nu), k(\nu)$, and $x(\nu)$ are ${\cal F}(Z(t)), {\cal F}(K(t))$, and ${\cal F}(X(t))$, respectively. Without loss of generality, in the discrete frequency domain, the gradient of $z$ under a specific frequency $\nu_i$ with respect to the $x$ under frequency $\nu_j$ is defined as
\begin{equation}
    \begin{aligned}
        & \frac{\partial z(\nu_i)}{\partial x(\nu_j)} = \frac{\partial }{{\partial x(\nu_j)}}( {k(\nu_i)x(\nu_i)}) = k(\nu_i){{\delta}_{\nu_j}(\nu_i)},
    \end{aligned}
    \label{Eq:conv_spec_grad}
\end{equation}
where ${\delta}_{\nu_j}(\nu_i)$ is the Kronecker delta function.
\begin{equation}
    \begin{aligned}
        {\delta}_{\nu_j}(\nu_i) = 
\begin{cases}
    1, & \text{if } \nu_i=\nu_j \\
    0, & \text{otherwise}
\end{cases}        
    \end{aligned}
\end{equation}

For the soft-plus function, it can be first expressed as Taylor series 
    \begin{equation}
    \begin{aligned}
        \sigma(Z(t)) &= \log ( {1 + {e^{Z(t)}}}) 
        \\&= \log (2) + \frac{1}{2}Z(t) + \frac{1}{8}{Z(t)^2} + O(Z(t)^4),
    \end{aligned}
\end{equation}
in which $Z(t)$ is small since the kernel $K(t)$ is small and $|X(t)| < 1$ by the assumption.
Hence, $O(Z(t)^4)$ becomes negligible. 
The Fourier transform of $\sigma(Z(t))$ is thus given as
\begin{equation}
    \begin{aligned}
        y(\nu) & = {\cal F}(\sigma( Z(t))) 
        \\ & \simeq {\cal F}( {\log ( 2 ) + \frac{1}{2}Z(t) + \frac{1}{8}{Z(t)^2}} )
        \\ & = 2\pi log(2) \delta_{0}(\nu)+\frac{1}{2} z(\nu)+\frac{1}{8} z(\nu)\otimes z(\nu),
    \end{aligned}
\end{equation}
and its spectral gradient is 
  \begin{equation}
    \begin{aligned}
        \frac{\partial y(\nu_i)}{\partial z(\nu_j)} & \simeq \frac{\partial }{\partial{z(\nu_j)}}( 2\pi log(2) \delta_{0}(\nu_i) 
        \\ & + \frac{1}{2} z(\nu_i)+\frac{1}{8} z(\nu_i)\otimes z(\nu_i) )
        \\ & = \frac{\partial }{\partial z(\nu_j)}(\frac{1}{2}z( {{\nu _i}} ) + \frac{1}{8} \sum_{r=0}^{n-1} z( {{\nu _i - \nu _r}} )z( {{\nu _r}} ))
        \\ & = \frac{1}{2}{{\delta}_{\nu _j}(\nu _i)} + \frac{1}{4}z( \nu_i - \nu_j),
    \end{aligned}
    \label{Eq:act_spec_grad}
\end{equation}
where $r$ is a dummy variable for the convolution and $n$ is the spectrum size of features. By Eq.~\ref{Eq:conv_spec_grad} and Eq.~\ref{Eq:act_spec_grad}, the spectral gradient of a convolutional layer in Eq.~\ref{Eq:conv_layer} is then written as
\begin{equation}
    \begin{aligned}
         \frac{\partial y(\nu_i)}{\partial x(\nu_j)} & = \sum _{q=0}^n \frac{{\partial y(\nu_i)}}{{\partial z(\nu_q )}}\frac{{\partial z(\nu_q)}}{{\partial x(\nu_j)}}
        \\ & \simeq \sum _{q=0}^n ((\frac{1}{2}{{\delta}_{\nu_q}(\nu_i)} + \frac{1}{4}z(\nu _i - \nu_q ))k( \nu_q){{\delta}_{\nu_j}(\nu_q)}) 
        \\ & = k( \nu_j)[\frac{1}{2}{{\delta}_{\nu_j}(\nu_i)} + \frac{1}{4}z(\nu_i-\nu_j)],
    \end{aligned}
    \label{Eq:convlayer_spec_grad_supp}
\end{equation}
where $i, j, \text{and~} q$ are the frequency indices.
Since $Z(t)$ is small as argued above, the corresponding spectrum $z(\nu)$ should also be small. We can therefore neglect the second term of Eq.~\ref{Eq:convlayer_spec_grad_supp},~\ie $\frac{1}{4}z(\nu_i-\nu_j)$, and approximate Eq.~\ref{Eq:convlayer_spec_grad_supp} as
\begin{equation}
    \begin{aligned}
        \frac{\partial y(\nu_i)}{\partial x(\nu_j)} \simeq \frac{1}{2}k( \nu_j){{\delta}_{\nu_j}(\nu_i)}
    \end{aligned}
\end{equation}
\end{proof}

\subsection{Gradient propagation for the frequency component of CE}\label{ssec:sgd_ce}
\begin{lemma}\label{lemma:ce_grad}
Given a convolutional layer that satisfies the assumption of lemma~\ref{lemma:conv_grad}. Let $x(\nu)$ denote the spectrum of input feature. 
For each semantic class $c$ in segmentation maps, let $k(\nu, c)$ and $y(\nu,c)$ denote the spectrum of kernel and that of the segmentation output, respectively. 
The spectral gradient for the frequency component of CE, $\mathcal{L}_{ce}(\nu_i)$ is 
\begin{equation}
    \begin{aligned}
        \frac{{\partial \mathcal{L}_{ce}(\nu_i)}}{{\partial x(\nu_j)}} 
         \simeq & \sum _c \frac{1}{2}k(\nu_j,c)[D_{0}(\nu_i) s(-\nu_j, c) - \delta_{\nu_j}(\nu_i)b(-\nu_i,c)],
    \end{aligned}
\end{equation}
where $\delta_{\nu_j}(\nu_i)$ is the Kronecker delta function and $D_{0}(\nu_i)$ is the Dirac delta function.
\end{lemma}

\begin{proof}
By lemma~\ref{lemma:conv_grad} and Eq.~\ref{Eq:spec_ce_component}, the spectral gradient $\frac{{\partial \mathcal{L}_{ce}(\nu_i)}}{{\partial x(\nu_j)}}$ is given as
\begin{equation}
    \begin{aligned}
        \frac{{\partial \mathcal{L}_{ce}(\nu_i)}}{{\partial x(\nu_j)}} 
        & = \sum_c \sum_q \frac{{\partial \mathcal{L}_{ce}(\nu_i)}}{{\partial y(\nu_q,c)}}\frac{{\partial y(\nu_q,c)}}{{\partial x(\nu_j)}}
        \\ & \simeq \sum_c \sum_q \frac{1}{2} \frac{{\partial \mathcal{L}_{ce}(\nu_i)}}{{\partial y(\nu_q,c)}} k( \nu_j,c){{\delta}_{\nu_j}(\nu_q)}
        \\ & = \sum_c \frac{1}{2} \frac{{\partial \mathcal{L}_{ce}(\nu_i)}}{{\partial y(\nu_j,c)}} k(\nu_j,c)
        \\ & = \sum_c \frac{1}{2}k(\nu_j,c) \frac
        {\partial    
        {~\sum _{\widetilde{c}} b(-\nu_i,\widetilde{c})({y_p}(\nu_i) - y(\nu_i,\widetilde{c}))}}
        {\partial y(\nu_j,c)} 
        \\ & = \sum _c \frac{1}{2}k(\nu_j,c)         [(\frac{\partial {y_p}(\nu_i)}{\partial y(\nu_j,c)} \sum \limits_{\widetilde{c}} b(-\nu_i,\widetilde{c})) \\ & -  (\sum \limits_{\widetilde{c}} \frac {\partial y(\nu_i,\widetilde{c})} {\partial y(\nu_j,c)}b(-\nu_i,\widetilde{c}))],
    \end{aligned}
\label{Eq:ce_spec_grad1}
\end{equation}
in which
\begin{equation}
    \begin{aligned}
        \frac{\partial y_p(\nu_i)}{\partial y(\nu_j,c)} 
        & = \int \frac{\partial y_p(\nu_i)}{\partial Y(t,c)}\frac{\partial Y(t,c)}{\partial y(\nu_j,c)} dt
        \\ & = \int \frac{\partial {\mathcal{F}}(Y_p(t))}{\partial Y(t,c)}\frac{\partial {\mathcal{F}}^{-1}(y(\nu,c))}{\partial y(\nu_j,c)} dt
        \\ & = \int \frac{\partial \int log({\sum _c} e^{Y(t,c)}) e^{-2j\pi\nu_i t}dt}{\partial Y(t,c)}
        \\ &\frac{\partial {\int y(\nu,c) e^{2j\pi\nu t}d\nu}}{\partial y(\nu_j,c)} dt
        \\ & = \int \frac{e^{Y(t,c)}}{{\sum _c} e^{Y(t,c)}} e^{-2j\pi(\nu_i-\nu_j) t}dt
        \\ & = \int S(t,c) e^{-2j\pi(\nu_i-\nu_j) t}dt
        = s(\nu_i-\nu_j, c),
    \end{aligned}
\label{Eq:amax_spec_grad}
\end{equation}
where $S(t,c)$ is the segmentation output after performing softmax on logits $Y$ and $s(\nu,c)$ is the spectrum of the segmentation output. 
Further, we have 
\begin{equation}
    \begin{aligned}
        \sum _{\widetilde{c}} b(-\nu_i,\widetilde{c}) = D_{0}(\nu_i)
    \end{aligned}
    \label{Eq:sum_b}
\end{equation}
by the Fourier transform of $\sum _{\widetilde{c}} B(t,\widetilde{c}) = 1$, \ie the fact that $B$ denote the probability distribution over semantic classes and should sum to one for each pixel.
Substituting Eq.~\ref{Eq:amax_spec_grad} and Eq.~\ref{Eq:sum_b} into Eq.~\ref{Eq:ce_spec_grad1}, we have the overall spectral gradient as 
\begin{equation}
    \begin{aligned}
        \frac{{\partial \mathcal{L}_{ce}(\nu_i)}}{{\partial x(\nu_j)}} 
         \simeq & \sum _c \frac{1}{2}k(\nu_j,c)
         \\ & [D_{0}(\nu_i) s(-\nu_j, c) - \delta_{\nu_j}(\nu_i)b(-\nu_i,c)]
    \end{aligned}
\end{equation}
\end{proof}

\subsection{Implementation details}\label{ssec:implement}
\noindent\textbf{Datasets.}
We examine the experiments upon the following three semantic segmentation datasets: PASCAL semantic segmentation benchmark \cite{everingham2015pascal}, DeepGlobe land-cover classification challenge \cite{demir2018deepglobe} and Cityscapes pixel-level semantic labeling task \cite{cordts2016cityscapes} (denoted as PASCAL, DeepGlobe and Cityscapes respectively). The PASCAL dataset contains 21 categories, 1464 training images, and 1449 validation images; the dataset further augmented by the extra annotations from \cite{BharathICCV2011}. The DeepGlobe dataset contains 7 categories, 803 training images, which are split into 701 and 102 images for training and validation, respectively. The Cityscapes dataset contains 19 categories, 2975 training images, and 500 validation images.

\noindent\textbf{Segmentation networks and implementation details.} 
In our experiment, we utilize the standard segmentation networks including DeepLab v3+~\cite{chen2018encoder} and Deep Aggregation Net (Dan)~\cite{kuo2018dan}. We adopt the ResNet-101~\cite{he2016deep} pre-trained on ImageNet-1k~\cite{russakovsky2015imagenet} as the backbone of these networks. These networks are trained by the following training policies: For all datasets, the images are randomly cropped to 513$\times$513 pixels; the training batch size are 8. For the PASCAL dataset, the network is trained with initial learning rate 0.0007 and 100 epochs; for DeepGlobe dataset, the network is trained with initial learning rate 0.007 and 600 epochs; for Cityscapes dataset, the network is trained with initial learning rate 0.001 and 200 epochs. For evaluation, the images are cropped to 513$\times$513 pixels for all datasets for consistent image size in spectral analysis.


\subsection{Experimental Data of Feature Truncation}\label{ssec:feat_detail}

\begin{table*}[h]
\centering
\begin{subtable}[h]{0.565\textwidth} 
    \resizebox{1.0\columnwidth}{!}{
    \begin{tabular}{ccccccc}
    \hline
    \multicolumn{1}{c|}{LRG} & \multicolumn{2}{c|}{PASCAL} & \multicolumn{2}{c|}{DeepGlobe} & \multicolumn{2}{c}{CityScapse} \\ \cline{2-7} 
    \multicolumn{1}{c|}{size} & mIoU & \multicolumn{1}{c|}{relative mIoU-drop} & mIoU & \multicolumn{1}{c|}{relative mIoU-drop} & mIoU & relative mIoU-drop \\ \hline
    \multicolumn{7}{c}{Baseline} \\ \hline
    \multicolumn{1}{c|}{129} & 78.5\% & \multicolumn{1}{c|}{0.0\%} & 54.3\% & \multicolumn{1}{c|}{0.0\%} & 67.8\% & 0.0\% \\
    \multicolumn{1}{c|}{97} & 78.4\% & \multicolumn{1}{c|}{0.1\%} & 54.2\% & \multicolumn{1}{c|}{0.2\%} & 65.8\% & 2.9\% \\
    \multicolumn{1}{c|}{81} & 78.3\% & \multicolumn{1}{c|}{0.3\%} & 54.1\% & \multicolumn{1}{c|}{0.3\%} & 64.5\% & 4.8\% \\
    \multicolumn{1}{c|}{65} & 78.1\% & \multicolumn{1}{c|}{0.5\%} & 54.0\% & \multicolumn{1}{c|}{0.5\%} & 62.7\% & 7.4\% \\
    \multicolumn{1}{c|}{49} & 77.4\% & \multicolumn{1}{c|}{1.3\%} & 53.8\% & \multicolumn{1}{c|}{0.9\%} & 59.4\% & 12.4\% \\
    \multicolumn{1}{c|}{33} & 75.6\% & \multicolumn{1}{c|}{3.6\%} & 53.4\% & \multicolumn{1}{c|}{1.7\%} & 53.3\% & 21.3\% \\
    \multicolumn{1}{c|}{17} & 68.6\% & \multicolumn{1}{c|}{12.5\%} & 52.0\% & \multicolumn{1}{c|}{4.1\%} & 38.6\% & 43.0\% \\ \hline
    \multicolumn{7}{c}{20\% PR for encoder} \\ \hline
    \multicolumn{1}{c|}{129} & 76.6\% & \multicolumn{1}{c|}{0.0\%} & 53.7\% & \multicolumn{1}{c|}{0.0\%} & 67.2\% & 0.0\% \\
    \multicolumn{1}{c|}{97} & 76.3\% & \multicolumn{1}{c|}{0.4\%} & 53.6\% & \multicolumn{1}{c|}{0.1\%} & 65.3\% & 2.8\% \\
    \multicolumn{1}{c|}{81} & 76.1\% & \multicolumn{1}{c|}{0.6\%} & 53.5\% & \multicolumn{1}{c|}{0.3\%} & 64.0\% & 4.7\% \\
    \multicolumn{1}{c|}{65} & 76.0\% & \multicolumn{1}{c|}{0.8\%} & 53.4\% & \multicolumn{1}{c|}{0.6\%} & 62.4\% & 7.1\% \\
    \multicolumn{1}{c|}{49} & 75.0\% & \multicolumn{1}{c|}{2.1\%} & 53.2\% & \multicolumn{1}{c|}{1.0\%} & 59.1\% & 12.1\% \\
    \multicolumn{1}{c|}{33} & 73.2\% & \multicolumn{1}{c|}{4.4\%} & 52.7\% & \multicolumn{1}{c|}{1.9\%} & 53.4\% & 20.5\% \\
    \multicolumn{1}{c|}{17} & 64.1\% & \multicolumn{1}{c|}{16.3\%} & 51.4\% & \multicolumn{1}{c|}{4.2\%} & 38.7\% & 42.4\% \\ \hline
    \multicolumn{7}{c}{40\% PR for encoder} \\ \hline
    \multicolumn{1}{c|}{129} & 74.4\% & \multicolumn{1}{c|}{0.0\%} & 53.0\% & \multicolumn{1}{c|}{0.0\%} & 66.1\% & 0.0\% \\
    \multicolumn{1}{c|}{97} & 74.2\% & \multicolumn{1}{c|}{0.2\%} & 52.9\% & \multicolumn{1}{c|}{0.1\%} & 64.3\% & 2.8\% \\
    \multicolumn{1}{c|}{81} & 74.1\% & \multicolumn{1}{c|}{0.4\%} & 52.8\% & \multicolumn{1}{c|}{0.3\%} & 63.0\% & 4.6\% \\
    \multicolumn{1}{c|}{65} & 73.9\% & \multicolumn{1}{c|}{0.7\%} & 52.7\% & \multicolumn{1}{c|}{0.6\%} & 61.5\% & 7.0\% \\
    \multicolumn{1}{c|}{49} & 73.4\% & \multicolumn{1}{c|}{1.4\%} & 52.5\% & \multicolumn{1}{c|}{1.0\%} & 58.2\% & 11.9\% \\
    \multicolumn{1}{c|}{33} & 72.1\% & \multicolumn{1}{c|}{3.1\%} & 52.0\% & \multicolumn{1}{c|}{1.9\%} & 52.8\% & 20.1\% \\
    \multicolumn{1}{c|}{17} & 66.8\% & \multicolumn{1}{c|}{10.2\%} & 50.8\% & \multicolumn{1}{c|}{4.1\%} & 39.8\% & 39.8\% \\ \hline
    \multicolumn{7}{c}{60\% PR for encoder} \\ \hline
    \multicolumn{1}{c|}{129} & 65.1\% & \multicolumn{1}{c|}{0.0\%} & 48.9\% & \multicolumn{1}{c|}{0.0\%} & 58.6\% & 0.0\% \\
    \multicolumn{1}{c|}{97} & 65.0\% & \multicolumn{1}{c|}{0.1\%} & 48.8\% & \multicolumn{1}{c|}{0.2\%} & 57.2\% & 2.4\% \\
    \multicolumn{1}{c|}{81} & 64.9\% & \multicolumn{1}{c|}{0.3\%} & 48.7\% & \multicolumn{1}{c|}{0.3\%} & 56.3\% & 4.0\% \\
    \multicolumn{1}{c|}{65} & 64.7\% & \multicolumn{1}{c|}{0.6\%} & 48.6\% & \multicolumn{1}{c|}{0.6\%} & 54.8\% & 6.5\% \\
    \multicolumn{1}{c|}{49} & 64.4\% & \multicolumn{1}{c|}{1.1\%} & 48.4\% & \multicolumn{1}{c|}{1.0\%} & 52.1\% & 11.0\% \\
    \multicolumn{1}{c|}{33} & 63.3\% & \multicolumn{1}{c|}{2.8\%} & 47.9\% & \multicolumn{1}{c|}{2.0\%} & 47.4\% & 19.1\% \\
    \multicolumn{1}{c|}{17} & 58.9\% & \multicolumn{1}{c|}{9.5\%} & 46.6\% & \multicolumn{1}{c|}{4.7\%} & 36.3\% & 38.0\%
    \end{tabular}    
    }
    \caption{DeepLab v3+}
    \label{table:feat_trunc:deeplab}
\end{subtable}

\begin{subtable}[h]{0.565\textwidth} 
    \resizebox{1.0\columnwidth}{!}{
    \begin{tabular}{ccccccc}
    \multicolumn{1}{c|}{LRG} & \multicolumn{2}{c|}{PASCAL} & \multicolumn{2}{c|}{DeepGlobe} & \multicolumn{2}{c}{CityScapse} \\ \cline{2-7} 
    \multicolumn{1}{c|}{size} & mIoU & \multicolumn{1}{c|}{relative mIoU-drop} & mIoU & \multicolumn{1}{c|}{relative mIoU-drop} & mIoU & relative mIoU-drop \\ \hline
    \multicolumn{7}{c}{Baseline} \\ \hline
    \multicolumn{1}{c|}{129} & 77.6\% & \multicolumn{1}{c|}{0.0\%} & 53.6\% & \multicolumn{1}{c|}{0.0\%} & 66.4\% & 0.0\% \\
    \multicolumn{1}{c|}{97} & 77.3\% & \multicolumn{1}{c|}{0.3\%} & 53.5\% & \multicolumn{1}{c|}{0.2\%} & 64.6\% & 2.7\% \\
    \multicolumn{1}{c|}{81} & 77.1\% & \multicolumn{1}{c|}{0.6\%} & 53.4\% & \multicolumn{1}{c|}{0.3\%} & 63.2\% & 4.8\% \\
    \multicolumn{1}{c|}{65} & 76.7\% & \multicolumn{1}{c|}{1.1\%} & 53.3\% & \multicolumn{1}{c|}{0.5\%} & 61.0\% & 8.1\% \\
    \multicolumn{1}{c|}{49} & 75.8\% & \multicolumn{1}{c|}{2.2\%} & 53.2\% & \multicolumn{1}{c|}{0.8\%} & 57.2\% & 13.9\% \\
    \multicolumn{1}{c|}{33} & 73.4\% & \multicolumn{1}{c|}{5.4\%} & 52.7\% & \multicolumn{1}{c|}{1.7\%} & 50.0\% & 24.8\% \\
    \multicolumn{1}{c|}{17} & 62.2\% & \multicolumn{1}{c|}{19.8\%} & 50.7\% & \multicolumn{1}{c|}{5.5\%} & 29.6\% & 55.4\% \\ \hline
    \multicolumn{7}{c}{20\% PR for encoder} \\ \hline
    \multicolumn{1}{c|}{129} & 76.8\% & \multicolumn{1}{c|}{0.0\%} & 53.6\% & \multicolumn{1}{c|}{0.0\%} & 65.8\% & 0.0\% \\
    \multicolumn{1}{c|}{97} & 76.5\% & \multicolumn{1}{c|}{0.4\%} & 53.4\% & \multicolumn{1}{c|}{0.3\%} & 64.1\% & 2.7\% \\
    \multicolumn{1}{c|}{81} & 76.3\% & \multicolumn{1}{c|}{0.6\%} & 53.3\% & \multicolumn{1}{c|}{0.5\%} & 62.6\% & 4.8\% \\
    \multicolumn{1}{c|}{65} & 75.9\% & \multicolumn{1}{c|}{1.2\%} & 53.2\% & \multicolumn{1}{c|}{0.7\%} & 60.3\% & 8.4\% \\
    \multicolumn{1}{c|}{49} & 75.0\% & \multicolumn{1}{c|}{2.3\%} & 52.9\% & \multicolumn{1}{c|}{1.3\%} & 56.7\% & 13.8\% \\
    \multicolumn{1}{c|}{33} & 72.8\% & \multicolumn{1}{c|}{5.2\%} & 52.4\% & \multicolumn{1}{c|}{2.2\%} & 49.0\% & 25.5\% \\
    \multicolumn{1}{c|}{17} & 61.7\% & \multicolumn{1}{c|}{19.6\%} & 50.1\% & \multicolumn{1}{c|}{6.6\%} & 28.3\% & 57.0\% \\ \hline
    \multicolumn{7}{c}{40\% PR for encoder} \\ \hline
    \multicolumn{1}{c|}{129} & 74.6\% & \multicolumn{1}{c|}{0.0\%} & 52.3\% & \multicolumn{1}{c|}{0.0\%} & 65.1\% & 0.0\% \\
    \multicolumn{1}{c|}{97} & 74.4\% & \multicolumn{1}{c|}{0.3\%} & 52.1\% & \multicolumn{1}{c|}{0.3\%} & 63.4\% & 2.6\% \\
    \multicolumn{1}{c|}{81} & 74.2\% & \multicolumn{1}{c|}{0.6\%} & 52.0\% & \multicolumn{1}{c|}{0.5\%} & 62.0\% & 4.7\% \\
    \multicolumn{1}{c|}{65} & 73.8\% & \multicolumn{1}{c|}{1.1\%} & 51.9\% & \multicolumn{1}{c|}{0.8\%} & 59.8\% & 8.1\% \\
    \multicolumn{1}{c|}{49} & 73.1\% & \multicolumn{1}{c|}{2.0\%} & 51.5\% & \multicolumn{1}{c|}{1.5\%} & 56.2\% & 13.6\% \\
    \multicolumn{1}{c|}{33} & 71.3\% & \multicolumn{1}{c|}{4.4\%} & 50.7\% & \multicolumn{1}{c|}{3.1\%} & 48.9\% & 24.9\% \\
    \multicolumn{1}{c|}{17} & 62.0\% & \multicolumn{1}{c|}{16.9\%} & 48.3\% & \multicolumn{1}{c|}{7.6\%} & 29.0\% & 55.4\% \\ \hline
    \multicolumn{7}{c}{60\% PR for encoder} \\ \hline
    \multicolumn{1}{c|}{129} & 65.6\% & \multicolumn{1}{c|}{0.0\%} & 49.5\% & \multicolumn{1}{c|}{0.0\%} & 57.0\% & 0.0\% \\
    \multicolumn{1}{c|}{97} & 65.4\% & \multicolumn{1}{c|}{0.3\%} & 49.5\% & \multicolumn{1}{c|}{0.0\%} & 55.5\% & 2.7\% \\
    \multicolumn{1}{c|}{81} & 65.3\% & \multicolumn{1}{c|}{0.5\%} & 49.5\% & \multicolumn{1}{c|}{0.1\%} & 54.2\% & 4.9\% \\
    \multicolumn{1}{c|}{65} & 65.1\% & \multicolumn{1}{c|}{0.8\%} & 49.4\% & \multicolumn{1}{c|}{0.2\%} & 52.4\% & 8.1\% \\
    \multicolumn{1}{c|}{49} & 64.5\% & \multicolumn{1}{c|}{1.6\%} & 49.3\% & \multicolumn{1}{c|}{0.5\%} & 49.0\% & 14.0\% \\
    \multicolumn{1}{c|}{33} & 63.0\% & \multicolumn{1}{c|}{4.0\%} & 48.9\% & \multicolumn{1}{c|}{1.2\%} & 42.8\% & 24.9\% \\
    \multicolumn{1}{c|}{17} & 54.9\% & \multicolumn{1}{c|}{16.3\%} & 47.9\% & \multicolumn{1}{c|}{3.3\%} & 26.7\% & 53.2\%
    \end{tabular}    
    }
    \caption{Dan}
    \label{table:feat_trunc:dan}
\end{subtable}

\caption{Results for feature truncation and network pruning on DeepLab v3+ and Dan. 
    We summarize the results of using 4 setups of network pruning: "Baseline" denote the experiment without SFP while "X PR for encoder" denotes that with SFP where X = (20\%, 40\%, and 60\%) are the pruning rates. 
    For each setup of network pruning, we further evaluate the results with 7 LRG sizes for feature truncation, \ie (129,97,81,65,49,33, and 17).}
\label{table:feat_trunc}
\end{table*}

\begin{figure*}[h]
\captionsetup[subfigure]{}
\centering
\begin{subfigure}[ht]{2\columnwidth}
    \centering
    \includegraphics[width=1.\columnwidth]{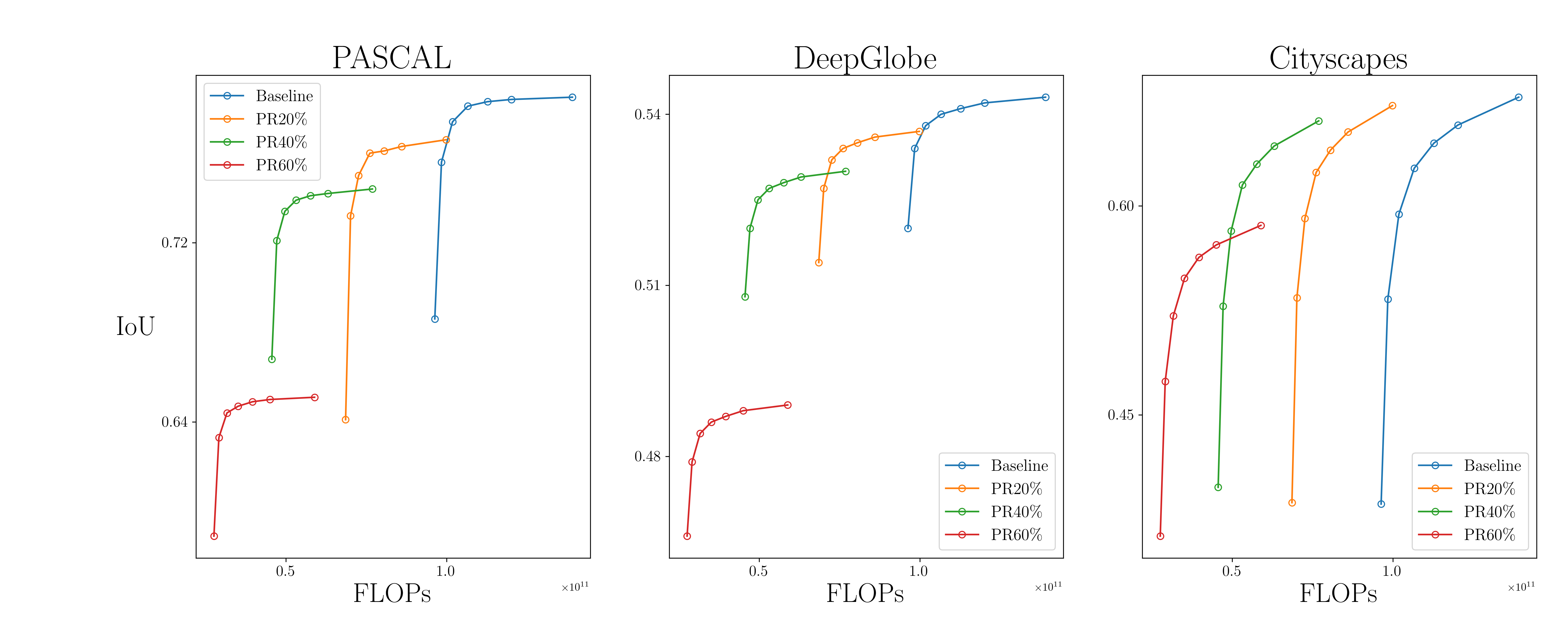}  
    \caption{DeepLab v3+}
    \label{Fig:iou_flops_deeplab}
\end{subfigure}
\\
\begin{subfigure}[ht]{2\columnwidth}
    \centering
    \includegraphics[width=1.\columnwidth]{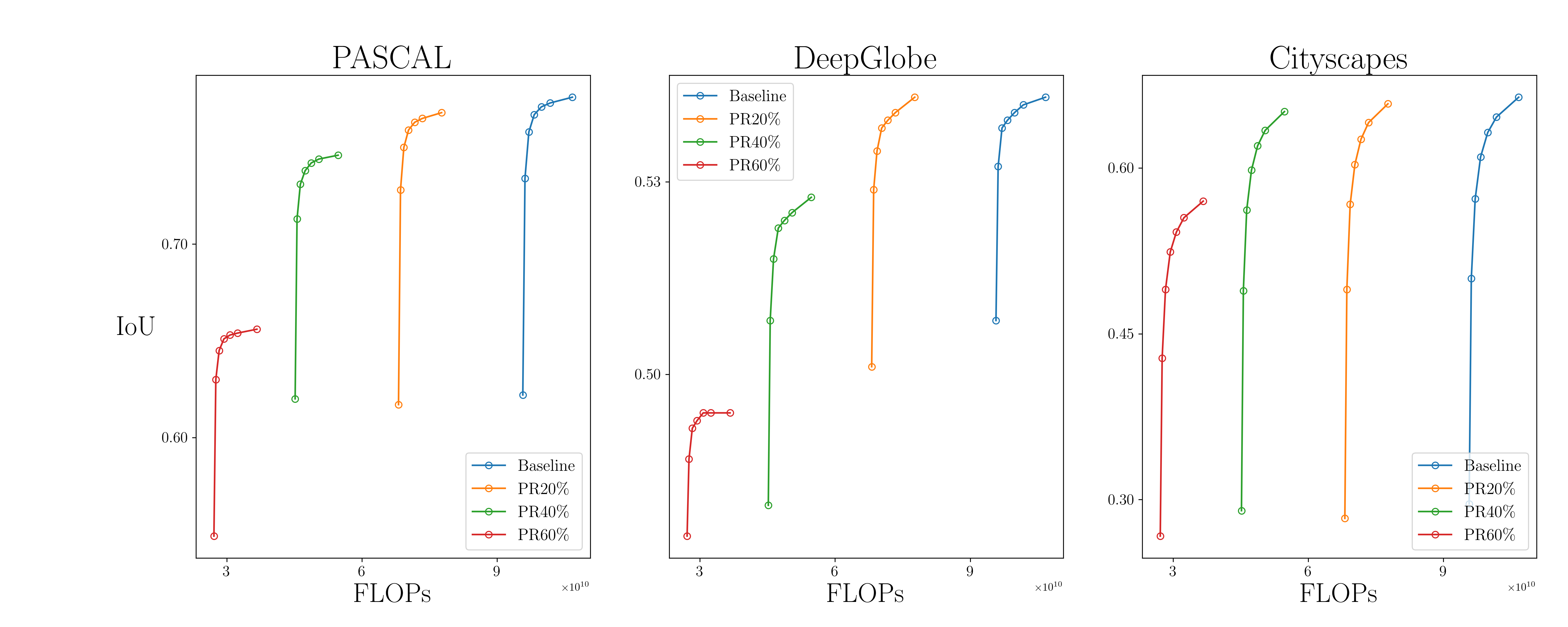}
    \caption{Dan}
    \label{Fig:iou_flops_dan}
\end{subfigure}
\caption{ 
IoU vs. FLOPs for DeepLab v3+ and Dan with various prune rates of SFP and LRG sizes for the feature truncation. Each line consists of 7 data with LRG sizes 129,97,81,65,49,33, and 17; the larger the LRG size, the higher the FLOPs.}
\label{Fig:iou_flops}
\end{figure*}

Evaluated DeepLab v3+ and Dan upon the PASCAL, DeepGlobe and Cityscapes datasets, table~\ref{table:feat_trunc} summarize the IoU score for each setup of feature truncation. Furthermore, Fig.~\ref{Fig:iou_flops} illustrate the plot of IoU score with respect to corresponding FLOPs (cf. table~\ref{table:feat_trunc_flops:flops}). 
For the experiment upon each dataset, the "mIoU" and "relative mIoU-drop" are evaluated, where "mIoU" is the mean IoU score over all semantic classes; "relative mIoU-drop" is the relative deduction rate of mIoU with respect to that of the model with same SFP setup and LRG size 129. 
Apparently from the tables, the mIoU decreases as either the PR increases or as the LRG size decreases. 
Closer inspection of these tables show that the relative mIoU-drop of the experiments on the Cityscapes dataset are significantly larger than that on the PASCAL and DeepGlobe datasets. Taking the "Baseline" model of DeepLab v3+ with LRG size 65 as an example, the relative mIoU-drop are 0.5\% and 0.7\% for PASCAL and DeepGlobe datasets, respectively, while becomes 7.4\% for the Cityscapes datasets. The same trends holds for the experiments with other SFP setup and LRG size based on DeepLab v3+. Also, the similar results are observed in the experiment on Dan as shown in Table~\ref{table:feat_trunc:dan}. 
These significant mIoU-drop on Cityscapes is well explained by the theoretical estimation $R_{ce}({\nu}_{max})$ in Table~\ref{table:spectral_stat}, where $R_{ce}({\nu}_{max})$ depict the relative loss rate of CE evaluated that positively correlates to relative mIoU-drop as discussed in section~\ref{ssec:method_spectral_ioU}. 

\section{Acknowledgement}\label{sec:acknow}
Authors acknowledge the support of the Ministry of Science and Technology of Taiwan (MOST110-2115-M-A49-003-MY2).

{\small
\bibliographystyle{aaai}
\bibliography{egbib}
}

\end{document}